\documentclass{article}
\usepackage[margin=1in]{geometry}
\usepackage[utf8]{inputenc} %
\usepackage[T1]{fontenc}    %

\usepackage{aux/libin_macros}
\usepackage{physics}
\title{A note on Linear Bottleneck Networks and their \\
Transition to Multilinearity}

\author{Libin Zhu$^{*\dagger}$ \quad Parthe Pandit$^\dagger$ \quad Mikhail Belkin$^\dagger$}
\date{}
\bibliography{aux/ref.bib}
\begin{document}

\maketitle

\begin{abstract}
  Randomly initialized wide neural networks transition to linear functions of weights as the width grows, in a ball of radius $O(1)$ around initialization. A necessary condition for this result is that all layers of the network are wide enough, i.e., all widths tend to infinity. However, the transition to linearity breaks down when this infinite width assumption is violated. In this work we show that linear networks with a bottleneck layer learn bilinear functions of the weights, in a ball of radius $O(1)$ around initialization. In general, for $B-1$ bottleneck layers, the network is a degree $B$ multilinear function of weights. Importantly, the degree only depends on the number of bottlenecks and not the total depth of the network.
  \end{abstract}

\blfootnote{\noindent $^*$Computer Science and Engineering,  University of California, San Diego}
\blfootnote{$^\dagger$Hal{\i}c{\i}o\u{g}lu Data Science Institute, University of California, San Diego}
\section{Introduction}

For a wide neural network~(WNN), when the network width is sufficiently large, there exists a linear function of parameters, arbitrarily close to the network function, in a ball of radius $O(1)$ in the parameter space around random initialization.
This local linearity explains the equivalence to the neural tangent kernel (NTK) regression for optimizing wide neural networks with small learning rates, first shown in~\cite{jacot2018neural}.
However, an important assumption for this \textit{transition to linearity}~\cite{liu2020linearity} to hold is that each layer must be sufficiently wide. If there is even one narrow ``bottleneck''  hidden layer, resulting in a so-called bottleneck neural network~(BNN), the work~\cite{liu2020linearity} showed that the transition to linearity does not occur. 

An immediate question at this point is, 
\begin{center}
    \textit{What functions of the weights does a neural network with a bottleneck layer represent}?
\end{center}

In this paper we answer this question for linear networks, \ie, networks with linear activations. We show that an arbitrary fixed depth linear network with $B-1$ bottleneck layers transitions to a $B$-th degree polynomial of weight parameters when the width of non-bottleneck layers approaches infinity.
Importantly, the degree of the polynomial is independent of the depth of the network and depends only on the number of bottleneck layers. 

Note that the network function, due to linear activations, is linear in each weight matrix. However, as a function of all weights together, it is a polynomial of degree equal to total number of layers. Our analysis of bottleneck networks shows that, as the width of non-bottleneck layers grows, the degree of this polynomial reduces to the number of bottleneck layers in the network plus 1, and becomes independent 
of the total number of layers. For wide networks without bottlenecks, this recovers the result of transition to linearity of wide neural networks.

In our technical analysis, for a BNN with $B-1$ bottleneck layers, we show that the spectral norm of the $B+1^{\rm st}$ derivative of the network function with respect to parameters, scales as with $1/\sqrt{m}$ where $m$ is the width of non bottleneck layers, but the spectral norm of the $B^{\rm th}$ derivative is $\Omega(1)$.  As a result, when $m$ goes to infinity, the network function transitions to a $B^{\rm th}$ order polynomial of the weights. We further strengthen this claim by showing this polynomial is in fact a multilinear function, where the network function is jointly linear in layer weights between consecutive bottleneck layers.

\subsection{Main contributions}

\begin{enumerate}
    \item Our first result \Cref{thm:bilinear} states that a 4 layer linear network with 1 bottleneck learns bilinear functions. We corroborate this via numerical experiments. See \Cref{sec:one_bottleneck}.
    \item We generalize the above result to show that linear networks with $B-1$ bottlenecks transition to $B^{\rm th} $degree polynomial functions, independent of the total number of layers. See \Cref{sec:multi}.
\end{enumerate}

\paragraph{Linearity of WNNs does not imply bilinearity of BNNs.} 
A cascade of two WNNs results in a BNN. While WNNs are known to transition to linearity for a large class of nonlinear activation functions, their cascade is in general not bilinear or quadratic.  Figure \ref{fig:Perturbation} provides a counter-example where a cascade of two  shallow WNNs with the \texttt{tanh} activation does not result in the network being a quadratic function of weights. We provide a detailed description below.

Consider two infinitely wide WNNs, $\x\mapsto f_1(\W_1;\x)$ and $\x\mapsto f_2(\W_2;\x)$ which can be approximated as a linear models,
\begin{align}
    f_1(\W_1;\vx) \approx \psi_1(\vx) +\inner{\W_1,\phi_1(\x)},\label{eq:approx1}\\
    f_2(\W_2;\vx) \approx \psi_2(\vx) +\inner{\W_2,\phi_2(\x)},\label{eq:approx2}
\end{align}
where  $\W_i$ are trainable parameters, $\x$ is the input, $\psi_i(\vx)$ are fixed offsets, and $\phi_i(\x)$ are feature maps.
The cascade of $f_1$ and $f_2$ is given by $\x\mapsto g(\W_1,\W_2;\x)$ where,
\begin{align}
    g(\W_1,\W_2;\vx)&:= f_2(\W_2;f_1(\W_1;\vx)) \approx
    \psi_2(f_1(\W_1;\x)) + \inner{\W_2,\phi_2(f_1(\W_1;\x))}\\
    &\approx \psi_2\Big({\psi_1(\x)+\inner{\W_1,\phi_1(\x)}}\Big) + \inner{\W_2,\phi_2\Big(\psi_1(\x)+\inner{\W_1,\phi_1(\x)}\Big)}\label{eq:approx3}
\end{align}
Note that $\psi_2$ and $\phi_2$ are, in general, nonlinear in their input. Hence $g$ is not bilinear in $(\W_1,\W_2)$.

\paragraph{Consequence of linear activations.}
However, when the activations of $f_1$ and $f_2$ are linear, we know that $\psi_2$ and $\phi_2$ are linear in their inputs.
Consequently, $g$ could be bilinear in $(\W_1,\W_2),$ if the approximation errors in equations (\ref{eq:approx1}-\ref{eq:approx3}) can be controlled together. 
Figure \ref{fig:Perturbation} details an experiment where this bilinearity is observed.

In this work, we rigorously prove that $g$ is well approximated with a bilinear model. We also provide a generalization to the case where a $B-1$ bottleneck network, which is a cascade of $B$ linear networks, is well approximated with a multilinear model with degree $B.$

\subsection{Background and related works}

Nearly linear training dynamics of Wide Neural Networks~(WNNs) allows an accurate analysis on the network evolution during training, especially with gradient descent and its variants ~\cite{du2018gradientdeep,liu2020loss,zou2019improved,allen2019convergence}. Besides, the complexity of WNNs  can be controlled in a ball of parameters with finite radius. Hence some generalization bounds were developed for shallow neural networks~\cite{arora2019fine} as well as deep ones~\cite{cao2019generalization}. This literature applies only to WNNs does not provide much insight into the behavior of bottleneck neural networks~(BNNs).

BNNs have been shown to be effective in many vision, speech and language tasks~\cite{yu2011improved,matejka2014neural,he2016deep}. Yet, we lack understanding for the training, optimization and generalization behavior of such networks. Some specific network architectures, such as  autoencoders~\cite{lecun2015deep}, are BNNs. 

Linear neural networks are non-linear in parameters and hence provide an analytically friendly way to draw insights about non-linear networks. For example, \cite{moroshko2020implicit} explored the transition between the ``kernel'' and ``non-kernel'' regimes on diagonal linear networks, which had been seen in wide non-linear neural networks~\cite{chizat2019lazy}.  The convergence of gradient descent was analyzed on deep linear neural networks in~\cite{arora2018convergence,du2019width,bartlett2018gradient}.  Linear networks also shed light on understanding implicit regularization~\cite{ji2018gradient,radhakrishnan2020alignment, gunasekar2018implicit, gidel2019implicit}  and matrix factorization~\cite{arora2019implicit,tarmoun2021understanding}. 

The optimization of bilinear models, though non-convex, has been well studied in previous works \cite{netrapalli2013phase, jain2013low} with several convergence guarantees. 
Our work connects BNNs to bilinear models, leaving open the question of optimization in learning BNNs. 

\begin{figure}[h!]
\centering
\begin{subfigure}[t]{0.49\textwidth}
\includegraphics[width=\linewidth]{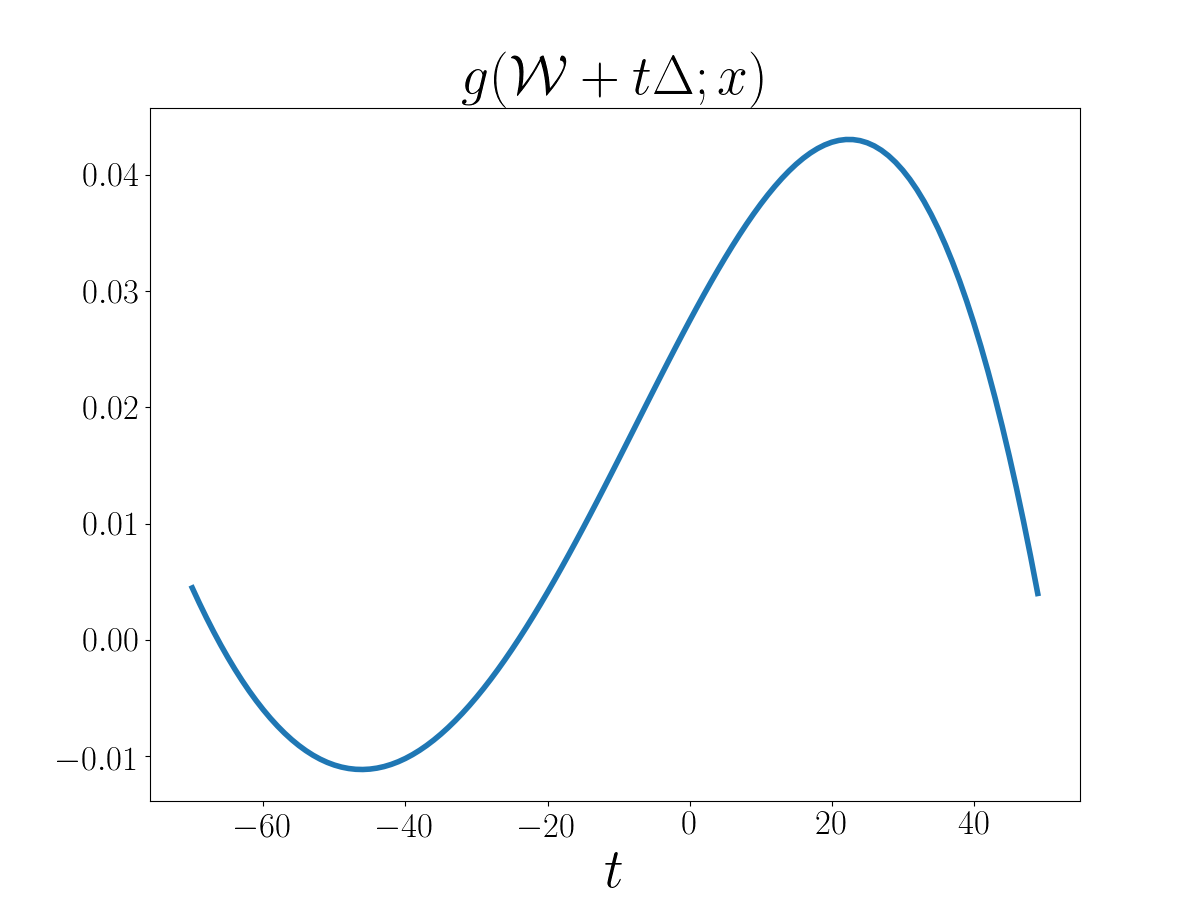}
     \caption{
     \label{fig:nonlinear}
     {Non-quadratic behaviour} with \texttt{tanh} activation.
     }
    \end{subfigure}
    \begin{subfigure}[t]{0.49\textwidth}
        \includegraphics[width=\linewidth]{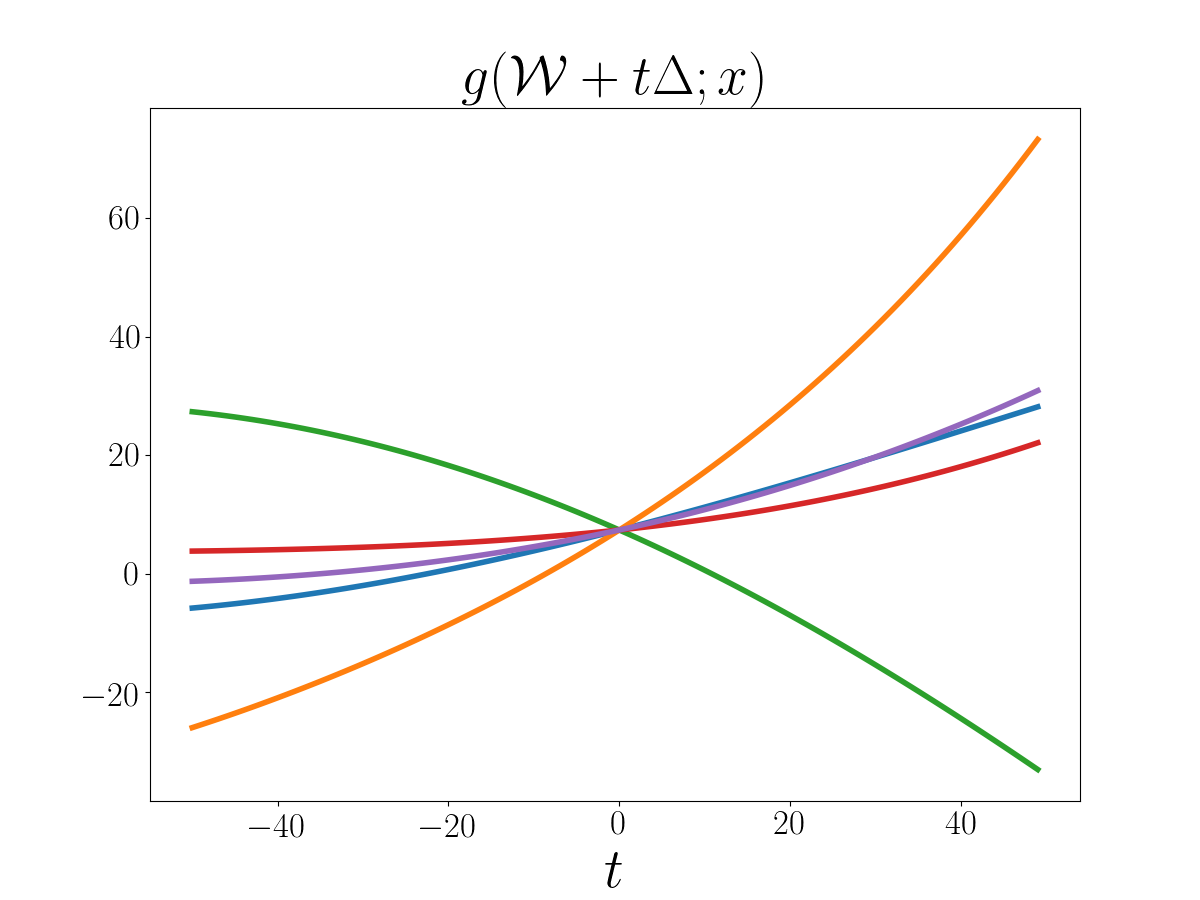}
    \caption{\label{fig:bilinear}
    {Quadratic behavior} with \texttt{identity} activation.
    }
    \end{subfigure}\\
    \begin{subfigure}[t]{0.49\textwidth}
        \includegraphics[width=\linewidth]{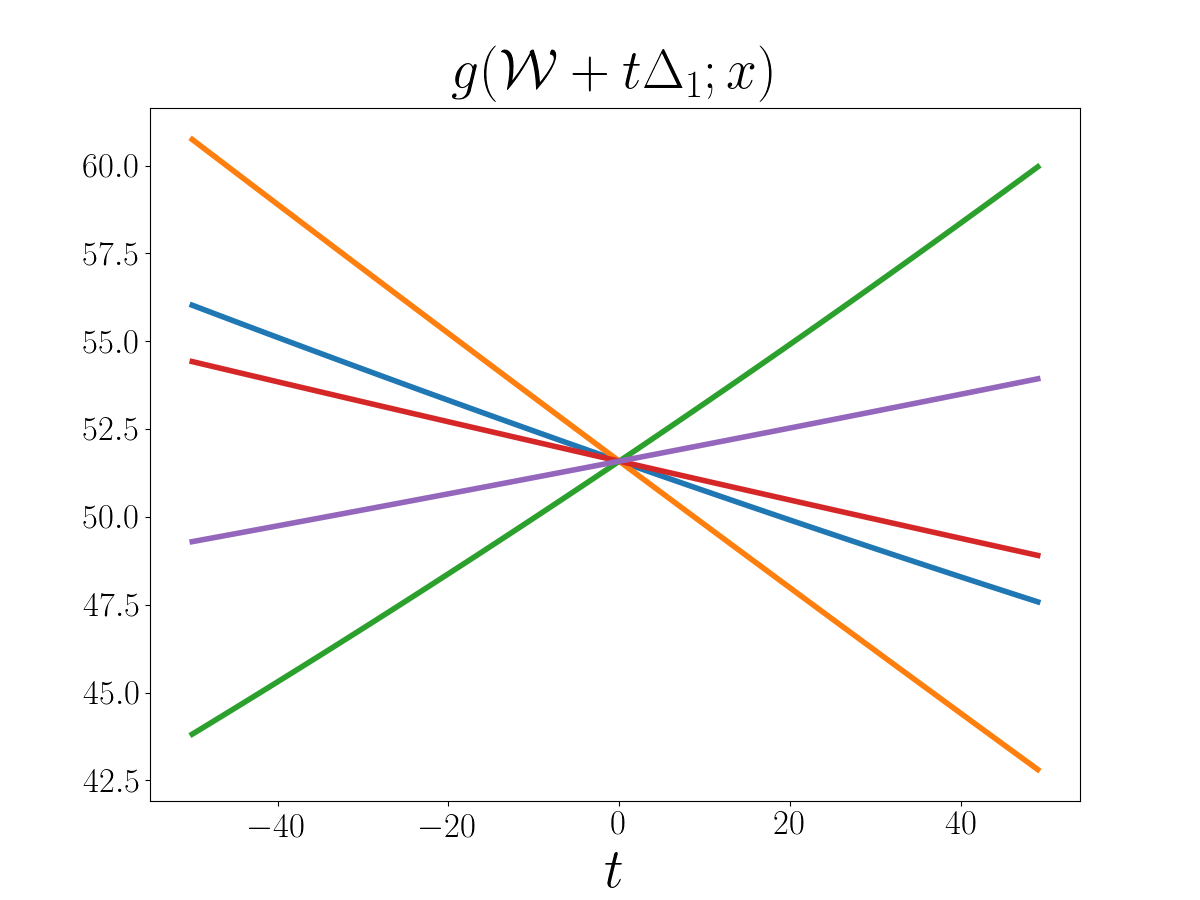}
        \caption{\label{fig:bilinear_1} Linearity in weights $\W_1$.}
    \end{subfigure}
        \begin{subfigure}[t]{0.49\textwidth}
        \includegraphics[width=\linewidth]{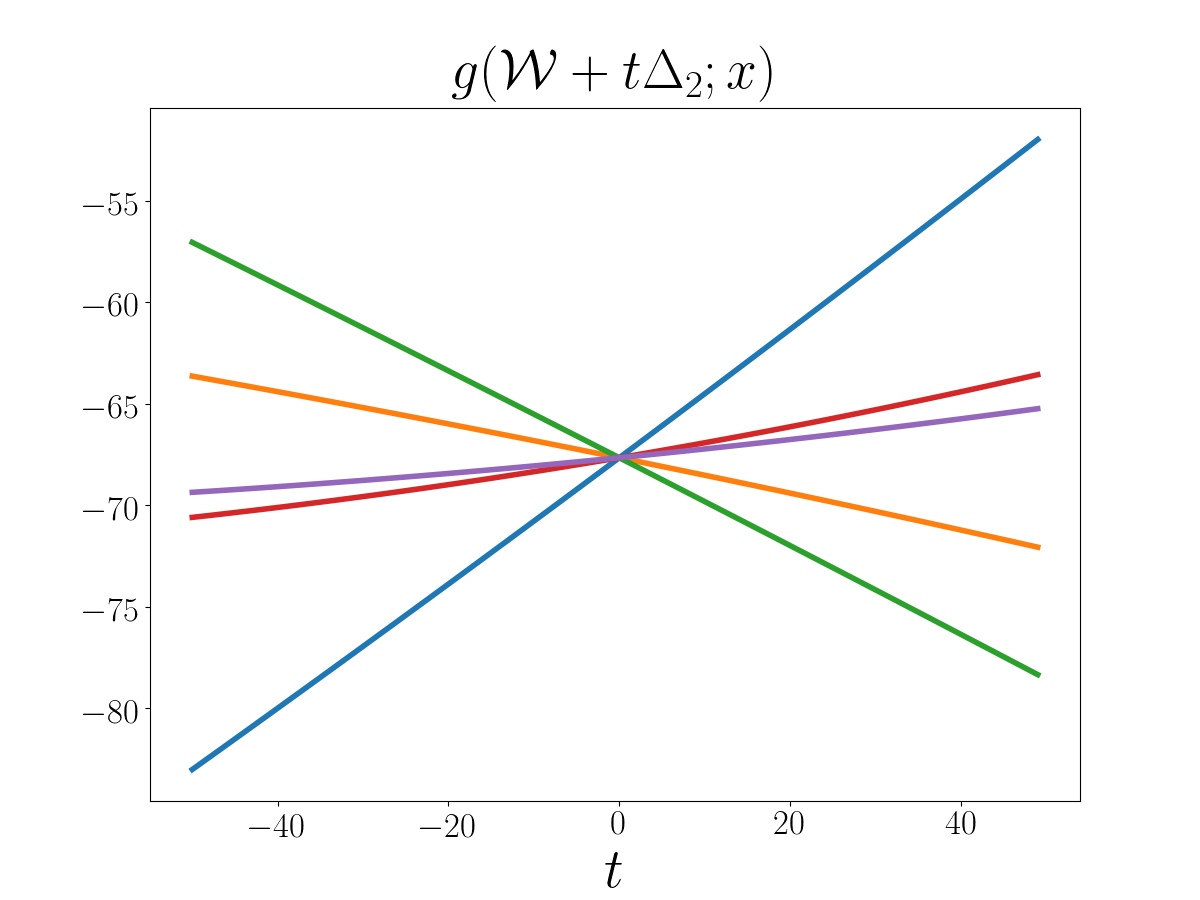}
        \caption{\label{fig:bilinear_2} Linearity in weights $\W_2$.}
    \end{subfigure}

\caption{\label{fig:Perturbation}
We consider a 4 layer neural network, with 1 input unit, 1 output unit, and hidden layer widths \texttt{[$10^4, 1, 10^4$]}, \ie, the network has 1 bottleneck layer. The network weights are $\Wc=(\W_1,\W_2),$ where $\W_1=(W_1^{(1)}, W_1^{(2)})$ are weights before the bottleneck, and $\W_2=(W_2^{(1)}, W_2^{(2)})$ are weights after.\newline
\textbf{Network function along a 1D subspace:} For a fixed input $\x$, we plot the output $g(\Wc+t\Delta,\x)$ v/s $t$, as the weights are varied along a randomly chosen direction $\Delta$ of the parameter space, under different settings. \newline
{ Figure (\ref{fig:nonlinear}): \texttt{tanh} activation.} The inflection point proves the network is not a quadratic function of weights.\newline
{
Figure (\ref{fig:bilinear}): \texttt{identity} activation.} The plot shows the network function along 5 directions. The network function exhibits quadratic behavior along each perturbation direction $\Delta$. Theorem \ref{thm:bilinear} provides a rigorous statement and proof of this quadratic behavior. 
In fact, this quadratic function is actually a bilinear function of $\Wc,$ which is linear in $\W_1$, and linear in $\W_2$, as demonstrated by figures (\ref{fig:bilinear_1}-\ref{fig:bilinear_2}).
\newline
{{Figure (\ref{fig:bilinear_1}): \texttt{identity} activation.}} Here we restrict perturbations to directions $\Delta_1$ that align with $(\W_1, \zero)$. The network function is linear in $\W_1$. A rigorous statement and proof is provided in Theorem \ref{thm:bilinear}.\newline
{ Figure (\ref{fig:bilinear_2}):} same setting as (\ref{fig:bilinear_1}) with directions $\Delta_2$ that align with $(\zero,\W_2)$. The network is linear in $\W_2$
}
\end{figure}

\section{Preliminaries}

We use  lowercase bold letters, $\rvw$, for vectors, uppercase letters, $W$, for matrices, uppercase bold letters, $\W$, for tuples of vectors or matrices, calligraphic uppercase letters $\mc{W}$ for tuples of tuples, and sans serif letters $\mathsf{W}$ for tensors. We denote the set $\{1,2,\cdots,n\}$ by $[n]$ for all $n\in\Natural$. We use $\inner{\cdot,\cdot}$ as the standard inner product for these objects, and the norms $\norm{\w},\norm{W},\norm{\W},\norm{\Wc},\norm{\mathsf{W}}$ to denote their respective Frobenius norms. We use $\opnorm{\mathsf{A}}$ to denote the operator norm of a linear operator represented as a tensor $\mathsf{A}$. We use $\otimes$ to denote the Kronecker product and outer product. Given a vector $\rvw\in\mathbb{R}^p$, we denote the Kronecker product of $n$ copies of $\rvw$ by $\rvw^{\otimes n}$ which is a rank-1 tensor of order $n$.

We denote by $\frac{\partial^k}{\partial \w_0^k} f(\w_0)$ the $k^{\rm th}$ derivative of $f(\w)$ evaluated at $\w_0$. This is a tensor of order $k$.

\paragraph{Degree-$n$ polynomial.} A function  $f(\rvw):\mathbb{R}^m\rightarrow \mathbb{R}$ is said to be a degree-$k$ polynomial if 
\begin{align}
    \opnorm{\frac{\partial^{k}f(\rvw)}{\partial\rvw^{k}}} > 0\quad\mathrm{and}\qquad\opnorm{\frac{\partial^{k+1}f(\rvw)}{\partial\rvw^{k+1}}} = 0, ~~~{\forall\w \in\mathbb{R}^m.}
\end{align}
Note that $\frac{\partial^{k}f(\rvw)}{\partial\rvw^{k}}$ is a tensor and the norm above is defined as in \cref{def:tensor_norm}.

\begin{definition}[Tuple product]\label{def:tuple_prod}
For a tensor $\tensor{A} \in \mathbb{R}^{r_1\times r_2\times\ldots\times r_k}$ and a tuple of vectors  $\V=(\rvv^1,\rvv^2,\ldots,\rvv^k)$ where $\rvv^i\in\mathbb{R}^{r_i}$ for $i\in[k]$, we denote
\begin{align}
    \innert{\tensor{A}, \V} = \sum_{i_k=1}^{r_k}\cdots\sum_{i_1=1}^{r_1} \tensor{A}_{i_1,i_2,\cdots,i_k}v^1_{i_1}v^2_{i_2}\ldots v^k_{i_k}.
\end{align}
Another way to interpret this is $\innert{\tensor{A}, \V} = \inner{\tensor{A},\tensor{V}}$ where $\tensor V = \rvv^1 \otimes \rvv^2\otimes\cdots\otimes \rvv^k $ is a rank-1 tensor. 
\end{definition}
\begin{definition}[Tensor spectral norm]
Given a tensor $\tensor{A} \in \mathbb{R}^{r_1\times r_2\cdots r_k}$, its spectral norm is defined as
\begin{align}\label{def:tensor_norm}
    \opnorm{\tensor{A}} = \sup_{\norm{\rvv^i} =1, \forall i\in[k] }\innert{\tensor{A},\V} = \sup_{\norm{\V} =1}\innert{\tensor{A},\V} = \sup_{\underset{\rank{\tensor{V}}=1}{\norm{\tensor V} =1}}\inner{\tensor{A},\tensor V},
\end{align}
where $\V=(\rvv^1,\rvv^2,\ldots,\rvv^k)$ and $\rvv^i\in\mathbb{R}^{r_i}$ for $i\in[k]$.
\end{definition}

\begin{definition}[2-norm ball]\label{def:2norm_ball}
The 2-norm ball $\mathbb{B}(\mc{V}_\init,R)$ is defined as:
\begin{align*}
    \mathbb{B}(\mc{V}_\init,R):=\left\{\mc{V} = (V_1,V_2,\ldots,V_k)\ \bigg|\  \norm{V_{i,\init} - V_i}\leq R, \forall i\in[k] \right\}
\end{align*}
\end{definition}

\begin{definition}[Taylor expansion]\label{def:taylor}
Consider a function $f(\rvw):\mathbb{R}^m \rightarrow \mathbb{R}$ and suppose $\frac{\partial^{n+1}}{\partial\rvw^{n+1}}f(\rvw)$ is continuous in the neighborhood  $\mathcal{U}\subset \mathbb{R}^m$ of a given point $\rvw_0$. Then in $\mathcal{U}$, $f(\rvw)$ can be written as
\begin{align}
    f(\rvw) = f(\rvw_0) + \sum_{k=1}^n\frac{1}{k!} \inner{\eval{\frac{\partial^k}{\partial\rvw^k} f(\rvw)}_{\rvw=\rvw_0}, (\rvw-\rvw_0)^{\otimes k}} + 
    R_{n+1}(\rvw,{\xivec}),
\end{align}
where $\xivec \in \mathcal{U}$.
\end{definition}

\begin{definition}[$n$-jet]\label{def:jet}
The $n$-jet of $f$ at the point $\rvw_0$ is defined to be the polynomial
\begin{align}
    (J_{\rvw_0}^n f)(\rvv)= f(\rvw_0) + \sum_{k=1}^n\frac{1}{k!} \inner{\eval{\frac{\partial^k}{\partial \w^k} f(\rvw)}_{\rvw=\rvw_0}, \rvv^{\otimes k}}.
\end{align}
\end{definition}

Therefore, there exists a $\xivec$ such that
\begin{align}
    \tag{Mean Value Theorem}
    \MoveEqLeft f(\rvw) =  (J_{\rvw_0}^n f)(\rvw-\rvw_0) + R_{n+1}(\rvw,{\xivec}),\\
    \MoveEqLeft R_{n+1}(\rvw, {\xivec}) := \frac{1}{(n+1)!} \inner{\eval{\frac{\partial^{n+1}}{\partial \w^{n+1}} f(\rvw)}_{\rvw=\xivec}, (\rvw-\rvw_0)^{\otimes n+1}}.\label{eq:def:remainder}
\end{align}

\begin{definition}[Multilinear function]\label{def:multilinear}
A function $$f(\rvw_1,\rvw_2,\cdots,\rvw_k):\mathbb{R}^{m_1}\times\mathbb{R}^{m_2}\times\cdots R^{m_k}\rightarrow \mathbb{R}^c$$  is said to be a multilinear function in $k$ variables, or $k$-linear function, if $\forall i\in[k]$, $f$ is linear in $\rvw_i$,  assuming all variables but $\rvw_i$ are held constant.
\end{definition}

\subsection{Bottleneck Networks}

We introduce some notations related to bottleneck neural networks. These are summarized in \Cref{fig:bnn} and detailed in the equations below.

\begin{figure}
    \centering
    \includegraphics[width=\linewidth]{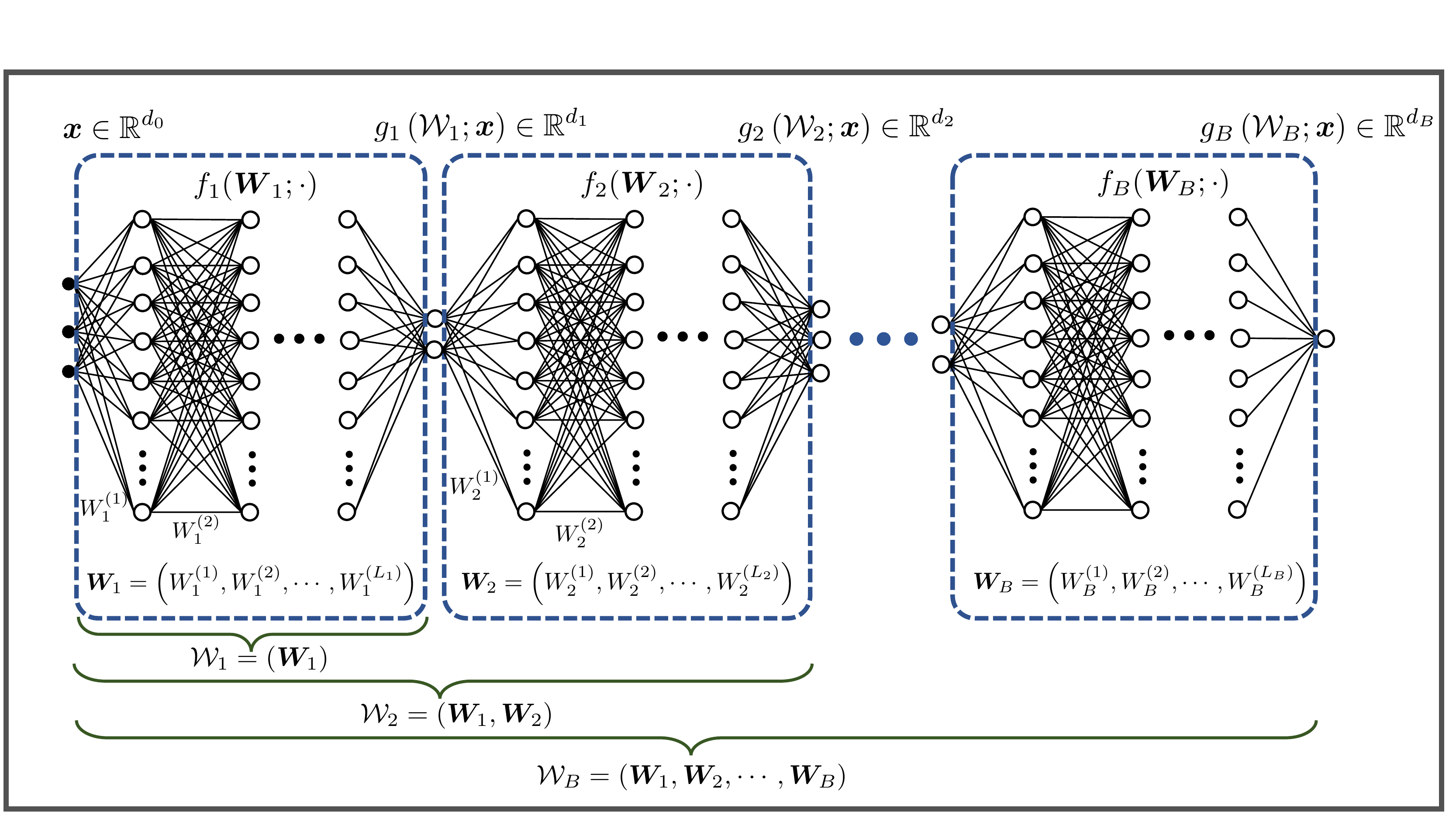}
    \caption{A bottleneck neural network with $B-1$ bottlenecks.}
    \label{fig:bnn}
\end{figure}

\begin{definition}[Wide Neural Network (WNN)]\label{def:wide_neural_network}
An $L$-layer WNN, $f(\rmW;\vx)\in\mathbb{R}^c$, with width $m$ is defined to be
\begin{align}
    f(\rmW;\vx) = \frac{1}{\sqrt{m}}W^{(L)}\frac{1}{\sqrt{m}}W^{(L-1)}...\frac{1}{\sqrt{d}}W^{(1)}\vx,
\end{align}
where $\rmW:=\left(W^{(1)},W^{(2)},...,W^{(L)}\right)$ are trainable parameters with $W^{(\ell)}\in\mathbb{R}^{m\times m}$ for $\ell\in\cbrac{2,3,\ldots,L}$,  $W^{(1)} \in \mathbb{R}^{m\times d}$,  $W^{(L)}\in\mathbb{R}^{c\times m}$, and $\vx \in \mathbb{R}^{d}$ is the input.
\end{definition}

A bottleneck network with $(B-1)$ bottlenecks can be viewed as a cascade of $B$ wide networks. Hence we define it in a recursive manner (see an illustration in Figure~\ref{fig:bnn}):
\begin{definition}[Bottleneck Neural Network (BNN)]\label{def:bottleneck_neural_network}
A BNN, $g(\mc W,\vx)$, with $B-1$ bottleneck layers is a recursion of $B$ WNNs $\{f_b\}_{b=1}^B$,
where each $f_b(\rmW_b;\cdot): \mathbb{R}^{d_{b-1}} \rightarrow \mathbb{R}^{d_b}$ is an $L_b$-layer WNN with weights $\rmW_b = \left(W_b^{(1)},\cdots ,W_b^{(L_b)}\right)$ for $i\in[L]$, and $\vx\in\mathbb{R}^d$ is the input. The trainable parameters of $g$ are,
\begin{align}
 \mc{W} = \mc{W}_B \qquad\text{where}\quad \mathcal{W}_b := {(\W_1,\W_2,\ldots,\W_b)\qquad b=1,2,\ldots,B}.
\end{align}
\end{definition}
\begin{remark}
Note that for $B=1,$ a BNN is simply a WNN. One can think of $g$ as a single neural network, however $d_b\ll m$ for all $b\in\{0,1,2,\ldots,B-1\}$, which justifies the name \textit{bottlenecks}.
\end{remark}

\paragraph{Network initialization.} In this paper, we initialize each trainable weight parameter of a neural network by drawing independent and identically distributed samples from the standard normal distribution $\mathcal{N}(0,1)$. This initialization is also referred as NTK initialization~\cite{jacot2018neural}.

By \Cref{def:2norm_ball}, the 2-norm ball for $\mc{W}$ around initialization is defined as:
\begin{align*}
    \mathbb{B}(\mathcal{W}_{\init},R)=\cbrac{\Wc=\rbrac{\rbrac{W_b^{(\ell_b)}}_{\ell_b=1}^{L_b}}_{b=1}^B\ \bigg|\ 
    \norm{W_b^{(\ell)} - W_{b,\init}^{(\ell)}}
    \leq R, \quad\ell\in[L_b],b\in[B].}
\end{align*}

\section{Linear network with one bottleneck}
\label{sec:one_bottleneck}

In this section, we consider a 4-layer linear network with 1 bottleneck layer {of width $r=O(1)$}, \ie, $B=2$ $L_1=L_2=2$ and $d_0=d, d_1=r, d_2=k$ in the notation described above. We show that this BNN transitions to a bilinear function of weights,  in a ball with radius $O({1})$ around random initialization, when the network width goes to infinity.

We consider the following network 
\begin{align}\label{eq:one_bottleneck}
        \MoveEqLeft g(\mc{W};\vx) = \frac{1}{\sqrt{m}}W_2^{(2)}\frac{1}{\sqrt{r}} W_2^{(1)} \frac{1}{\sqrt{m}}W_1^{(2)}\frac{1}{\sqrt{d}} W_1^{(1)}\vx,\\
        \MoveEqLeft  \mc{W}:=(\W_1,\W_2),\text{ with } \W_1:=\rbrac{W_1^{(1)},W_1^{(2)}}, \text{ and } \W_2:=\rbrac{W_2^{(1)},W_2^{(2)}},
\end{align}
where $W_1^{(1)}\in\mathbb{R}^{m\times d}, W_1^{(2)}\in\mathbb{R}^{r\times m}, W_2^{(1)} \in \mathbb{R}^{m\times r}, W_2^{(2)} \in \mathbb{R}^{ k\times m}$ and $\vx\in\mathbb{R}^d$ is the input. Here the input dimension is $d$, the bottleneck width is $r$ and the output dimension is $k$. In this section, we assume $k=1$. We note that the results can be easily extended to multiple output case, i.e., $k>1$.

Alternatively, we can regard $g(\mc{W};\vx)$ as a cascade of two wide neural networks $f_1$ and $f_2$, where the output of $f_1$ serves as the input of $f_2$. Specifically, for any $\vx\in\mathbb{R}^d, \vz\in\mathbb{R}^r$, we can define
\begin{align*}
    f_1(\rmW_1;\vx) = \frac{1}{\sqrt{m}}W_1^{(2)}\frac{1}{\sqrt{d}}W_1^{(1)}\vx,~~f_2(\rmW_2;\vz) = \frac{1}{\sqrt{m}}W_2^{(2)}\frac{1}{\sqrt{r}}W_2^{(1)}\vz,
\end{align*}
where $\rmW_1 :=\left(W_1^{(1)},W_1^{(2)}\right)$ and $\rmW_2 :=\left(W_2^{(1)},W_2^{(2)}\right)$.

Then it is not hard to see that
\begin{align}
    g(\mc{W};\vx) = f_2\left(\rmW_2;f_1\left(\rmW_1;\vx\right)\right).
\end{align}

Observe that $g(\mc{W};\x)$ is a $4^{\rm th}$ degree polynomial in $\mc{W}$, as it is a $4$ layer network function. However, if the network width is sufficiently large, \ie, $m=\omega(1)$, we show that the degree is approximately only $2$.

To state our result, we define the following derivatives at initialization,
\begin{align}
    \bm{g}_\init:= \frac{\partial g(\Wc,\x)}{\partial\Wc}\bigg|_{\mc{W}_{\init}}\qquad\qquad
    \bm{H}_\init := \frac{\partial^2 g(\Wc,\x)}{\partial\W_1\partial\W_2}\bigg|_{\Wc=\Wc_{\init}}
\end{align}
Our main result shows the transition of the network function to a bilinear form. 
\begin{theorem}[Transition to bilinearity]\label{thm:bilinear}
Given a ball $\mathbb{B}(\mc{W}_{\init},R)$ with $R = O(1)$, with probability at least $1-12e^{-\nicefrac{m}{32}}$, for a fixed input $\vx\in\mathbb{R}^d$, the function $\mc W\mapsto g(\mc{W};\vx)$ can be accurately approximated by a bilinear function $q(\Wc)$, 
\begin{align}
    \MoveEqLeft\left|g(\mc{W};\x) - q(\Wc)\right|= O\left(\frac{1}{\sqrt{m}}\right),\\
    \MoveEqLeft q(\Wc) :=  g(\mc{W}_{\init};\vx) + 
    \bm{g}_\init\tran 
    (\mc{W}-\mc{W}_{\init}) + (\W_1-\bm{W}_{1,\init})\tran\bm{H}_\init(\W_2-\bm{W}_{2,\init})
\end{align}
where $q(\mc W)$ is linear in $\W_1$ and linear in $\W_2$.
\end{theorem}
\begin{proof}[Proof of \Cref{thm:bilinear}]
Our proof is based on an application of the mean value theorem. Given $R = O(1)$, the Taylor expansion with a Lagrange remainder term, assures that there exists a $\xivec \in \mathbb{B}(\mc{W}_{\init},R)$ such that
\begin{align*}
    g(\mc{W};\x) = J_{\mc{W}_\init}^2 \{g(\cdot,\vx)\}\left(\mc{W}-\mc{W}_\init\right) + R_3(\mc{W},\xivec),
\end{align*}
where $J_{\mc{W}_\init}^2 g(\cdot;\vx))$ is the $2$-jet of $g$, defined in \cref{def:jet}, and $R_3(\mc{W},\xivec)$ is the remainder term defined in \cref{eq:def:remainder}. 

To prove the result we first show that the second term above is negligible, \ie, $|R_3(\Wc,\xivec)| = O(m^{-\half})$ , using a Cauchy-Schwarz inequality. 
We then show that the first term is well approximated by the bilinear form in the statement of the result, \ie,
\begin{align}
   \norm{J^2_{\Wc_\init}\{g(\cdot,\vx)\}(\Wc-\Wc_\init) - q(\Wc)} = {O(m^{-\half})}.
\end{align} 

For multiple output case, i.e., $k>1$, we can apply the same analysis to bound $R_3$ for each output component. Therefore, each component of the output, i.e., $g_i$ for $i\in[k]$ will transition to a bilinear function. 

The intermediate results are stated in the following lemma.
\begin{lemma} \label{lem:bilinear_bounds} The following statements hold for randomly initialized $\Wc_\init.$
\begin{enumerate}[label=(\alph*)]
    \item 
\label{lem:bilinear_upper_bound}
With probability at least $1-8e^{-\nicefrac{m}{2}}$, we have 
\begin{align}
    |R_3(\mc{W},\xivec)|
    \leq  \frac{(8\sqrt{m} + 2\sqrt{r}+\sqrt{d}+1 +4R)\|\vx\|}{m\sqrt{rd}}R^3,\qquad{\forall\xivec\in\mathbb{B}(\mc W_{\init},R).}
\end{align}
\item\label{lem:bilinear_lower_bound}
With probability at least $1-4e^{-\nicefrac{m}{32}}$ over random initialization, we have 
\begin{align}
    \norm{{\frac{\partial^2 g(\mc W;\x)}{\partial \mc W^2}\bigg|_{\Wc_\init}} - \begin{bmatrix}\zero &\H_\init \\
    \H_\init\tran & \zero
    \end{bmatrix}}_{\rm op} \leq\frac{2(\sqrt{6}+\sqrt{d})\|\vx\|\log m}{\sqrt{md}}= \tilde{O}\left(\frac{1}{\sqrt{m}}\right),
\end{align}
where
\begin{align}
    \opnorm{\H_\init}  \geq \frac{\norm{\x}}{24\sqrt{rd}}.
\end{align}
\end{enumerate}
\end{lemma}
{Note that if $\|\x\| = 0$, $g(\mc{W};\vx) = 0$ hence we can find $q(\mc{W}) = 0$ which is a bilinear function. Otherwise, $\opnorm{\H_\init} = \Omega(1)$ with high probability, which justifies the quadraticity of $q(\mc{W})$.}

Together this yields the result using triangle inequality. This concludes the proof of \Cref{thm:bilinear}.
\end{proof}

Before proceeding to the proofs of the intermediate lemma used in proving \cref{thm:bilinear}, we pause to visualize the bilinear model. We consider 15 different perturbations from the  initial weights. The first 5 perturbations apply to all weights of the network, as shown in \Cref{fig:bilinear}. The next 5 perturbations are applied to only the first block of the network, \ie, before the bottleneck (\cref{fig:bilinear_1}), and the last 5 are applied to the weights after the bottleneck (\cref{fig:bilinear_2}).

Observe that while the network function is quadratic in the weights, it is in fact linear in $\W_1$ and $\W_2$ the two blocks of weights separated by the bottleneck.

\begin{remark}
We note that the class of bilinear functions is a proper subset of the class of quadratic functions, which is important for the complexity of the class of bottleneck neural networks. 
\end{remark}
\begin{remark}
Optimization of bilinear models can be much easier than that of general quadratic models~\cite{netrapalli2013phase}.
\end{remark}

\subsection{Proof of \texstr{\Cref{lem:bilinear_bounds}}}

It is instructive to note the derivatives of the network function with respect to the weight matrices.
\begin{align*}
    \nabla g(\mc W;\x) &= \left[\frac{\partial g}{\partial W_{1}^{(1)}},\frac{\partial g}{\partial W_{1}^{(2)}},\frac{\partial g}{\partial W_{2}^{(1)}},\frac{\partial g}{\partial W_{2}^{(2)}}\right]\tran  
    =\frac{1}{m} \begin{bmatrix}{\mathrm{vec}((W_{2}^{(2)}}  W_{2}^{(1)}{W_{1}^{(2)}} )\otimes  \vx)\\
    \mathrm{vec}(({W_{2}^{(2)}}  W_{2}^{(1)}) \otimes (W_{1}^{(1)}\vx))\\
    \mathrm{vec}(  {W_{2}^{(2)}}\otimes ( W_{1}^{(2)} {W_{1}^{(1)}} \vx))\\
 \mathrm{vec}(\mathbf{1} \otimes  W_{2}^{(1)} W_{1}^{(2)} {W_{1}^{(1)}} \vx)\end{bmatrix}\in\mathbb{R}^{m(d+2r+1)}.%
\end{align*}

\begin{proof}[Proof of {\Cref{lem:bilinear_bounds}\ref{lem:bilinear_upper_bound}}]
By Lemma~\ref{lemma:tensor_norm}, the norm of $\frac{\partial^3}{\partial \Wc^3} g(\mc{W})$ can be bounded by the summation of the norm of blocks with respect to the weights in each layer. Specifically, 
\begin{align*}
    \opnorm{\frac{\partial^3}{\partial \Wc^3} g(\mc{W})} &\leq 6\opnorm{\frac{\partial^3 g}{\partial W_1^{(1)}\partial W_1^{(2)}\partial W_2^{(1)}}} + 6\opnorm{\frac{\partial^3 g}{\partial W_1^{(1)}\partial W_1^{(2)}\partial W_2^{(2)}}} + 6\opnorm{\frac{\partial^3 g}{\partial W_1^{(1)}\partial W_2^{(1)}\partial W_2^{(2)}}}\\
    &~~~~+ 6\opnorm{\frac{\partial^3 g}{\partial W_1^{(2)}\partial W_2^{(1)}\partial W_2^{(2)}}}.
\end{align*}

Take the first quantity on the RHS of the above inequality as an example:
\begin{align*}
    \opnorm{\frac{\partial^3 g}{\partial W_{1}^{(1)}\partial W_{1}^{(2)}\partial W_{2}^{(1)}}} &= \sup_{\norm{V_1^{(1)}}=\norm{V_1^{(2)}}=\norm{V_2^{(1)}} = 1}\norm{ \frac{1}{\sqrt{m}}W_2^{(2)}\frac{1}{\sqrt{r}} V_2^{(1)} \frac{1}{\sqrt{m}}V_1^{(2)}\frac{1}{\sqrt{d}} V_1^{(1)}\vx}\\
    &\leq \frac{\|\vx\|}{m\sqrt{rd}} \opnorm{W_2^{(2)}}
    \leq \frac{\|\vx\|}{m\sqrt{rd}} \left(\opnorm{W_{2,\init}^{(2)}}+R\right).
\end{align*}

Note that $W_{2,\init}^{(2)}$ is a random matrix with i.i.d. standard normal entries. Therefore we use a concentration inequality, detailed in  Lemma~\ref{lemma:random_matrix}, to bound its norm. For the choice $t=\sqrt{m}$, we have that, $\opnorm{W_{2,\init}^{(2)}}\leq 2\sqrt{m}+1$, with probability at least $1-2e^{-m/2}$. Thus,
\begin{align*}
    \opnorm{\frac{\partial^3 g}{\partial W_{1}^{(1)}\partial W_{1}^{(2)}\partial W_{2}^{(1)}}} \leq \frac{(2\sqrt{m}+1+R)\|\vx\|}{m\sqrt{rd}},\qquad \text{w.p.}\geq 1-2e^{-\nicefrac{m}{2}}.
\end{align*}

We can similarly bound the rest of the blocks, and apply union bound over the randomness of the four layers. This yields that with probability at least $1-8e^{-\nicefrac{m}{2}}$, we have,
\begin{align*}
    \opnorm{\frac{\partial^3}{\partial \Wc^3} g(\mc{W})} &\leq 12 \frac{(2\sqrt{m}+\sqrt{r}+R)\|\vx\|}{m\sqrt{rd}} + 6\frac{(2\sqrt{m}+\sqrt{d}+R)\|\vx\|}{m\sqrt{rd}} + 6\frac{(2\sqrt{m}+1+R)\|\vx\|}{m\sqrt{rd}}\\
    &= \frac{6(8\sqrt{m} + 2\sqrt{r}+\sqrt{d}+1 +4R)\|\vx\|}{m\sqrt{rd}}
\end{align*}

Next, for any $\xivec \in \mathbb{B}(\mc{W}_\init,R)$ with $R = O(1)$, a Cauchy–Schwarz inequality yields,
\begin{align*}
    |R_3(\mc{W},\xivec)|&\leq \max_{\xivec \in \mathbb{B}(\mc{W}_{\init},R)}  \frac{1}{6} \inner{\eval{\frac{\partial^3}{\partial \Wc^3} g(\mc{W})}_{\mc{W}=\xivec}, (\mc{W}-\mc{W}_{\init})^{\otimes 3}} \\
    &\leq \frac{1}{6}\max_{\xivec \in \mathbb{B}(\mc{W}_{\init},R)} \opnorm{ \eval{\frac{\partial^3}{\partial \Wc^3} g(\mc{W})}_{\mc{W}=\xivec}} \norm{\mc{W}-\mc{W}_{\init}}^3\\
    &\leq \frac{(8\sqrt{m} + 2\sqrt{r}+\sqrt{d}+1 +4R)\|\vx\|}{m\sqrt{rd}}R^3 = O\left(\frac{1}{\sqrt{m}}\right),\qquad\text{w.p.}\geq 1-8e^{-\nicefrac{m}{2}}.
\end{align*}

This concludes the proof of {\Cref{lem:bilinear_bounds}\ref{lem:bilinear_upper_bound}}.
\end{proof}

\begin{proof}[Proof of \Cref{lem:bilinear_bounds}\ref{lem:bilinear_lower_bound}]

We will show  $\eval{\frac{\partial^2}{\partial \Wc^2} g(\mc{W})}_{\Wc_\init}$ is close to  $\begin{bmatrix}\zero &\H_\init \\
    \H_\init\tran & \zero
    \end{bmatrix}$. Notice that the off-diagonal blocks are close by definition.
    
For the upper left block, i.e.,$\begin{bmatrix}\frac{\partial^2 g}{\left(\partial W_1^{(1)}\right)^2}  &\frac{\partial^2 g}{\partial W_1^{(1)} \partial W_1^{(2)}}  \\
   \frac{\partial^2 g}{\partial W_1^{(2)} \partial W_1^{(1)}}  & \frac{\partial^2 g}{\left(\partial W_1^{(2)}\right)^2} 
    \end{bmatrix}$, it is not hard to see that
\begin{align*}
   \frac{\partial^2 g_i}{\left(\partial W_1^{(1)}\right)^2} =\zero,\quad \frac{\partial^2 g_i}{\left(\partial W_1^{(2)}\right)^2} = \zero.
\end{align*}
Further, by definition, we have,
\begin{align}\nonumber
\opnorm{\eval{\frac{\partial^2 g}{\partial W_1^{(1)} \partial W_1^{(2)}}}_{\mc{W} = \mc{W}_{\init}}}&= \opnorm{\eval{\frac{\partial^2 g}{\partial W_1^{(2)} \partial W_1^{(1)}}}_{\mc{W} = \mc{W}_{\init}}} \\\nonumber
&=\sup_{\norm{V_1^{(1)}}=\norm{V_1^{(2)}}=1}\opnorm{ \frac{1}{\sqrt{m}}W_{2,\init}^{(2)}\frac{1}{\sqrt{r}} W_{2,\init}^{(1)} \frac{1}{\sqrt{m}}V_1^{(2)}\frac{1}{\sqrt{d}} V_1^{(1)}\vx}\\\label{eq:hessian_layer1}
&\leq \frac{1}{m\sqrt{rd}}\opnorm{W_{2,\init}^{(2)} W_{2,\init}^{(1)}}\|\vx\|,
  \end{align}  
using Cauchy-Schwarz inequality.

Note that 
    $W_{2,\init}^{(2)} \left(W_{2,\init}^{(1)}\right)_{[:,j]}$
 is a Gaussian random variable with variance $\norm{\left(W_{2,\init}^{(1)}\right)_{[:,j]}}^2$ conditioned on $W_{2,\init}^{(1)}$.
Note that $\norm{\left(W_{2,\init}^{(1)}\right)_{[:,j]}}^2 \sim\chi^2(m)$. We apply Lemma~\ref{lemma:chi_2} and choose $t=\frac{1}{2}$ to get, 
\begin{align}
\frac{m}{2}\leq \norm{\left(W_{2,\init}^{(1)}\right)_{[:,j]}}^2\leq \frac{3m}{2},\qquad
\text{ w.p.}\geq 1-2e^{\nicefrac{-m}{32}}   .
\end{align}

We can bound $W_{2,\init}^{(2)} \left(W_{2,\init}^{(1)}\right)_{[:,j]}$, by setting $t = \frac{\sqrt{6m}}{2}\log m$ in Lemma~\ref{lemma:gaussian}, which gives
\begin{align}
   \left| W_{2,\init}^{(2)} \left(W_2^{(1)}\right)_{[:,j]}\right| \leq \frac{\sqrt{6m}}{2}\log m,\qquad \text{w.p.}\geq 1-2e^{-\nicefrac{\log^2 m}{2}}.
\end{align}

Applying union bound over the indices $j\in[r]$, we have, 
\begin{align}
    \opnorm{W_{2,\init}^{(2)} W_{2,\init}^{(1)}} \leq \sqrt{r}\frac{\sqrt{6m}\log m}{2},\qquad\wp\geq 1-2re^{-\nicefrac{m}{32}} - 2e^{-\nicefrac{\log^2 m}{2}}.
\end{align}
The above bound along with equations \eqref{eq:hessian_layer1} gives the following high-probability bound, 
\begin{align*}
    \opnorm{\eval{\frac{\partial^2 g}{\partial W_1^{(1)} \partial W_1^{(2)}}}_{\mc{W} = \mc{W}_{\init}}} \leq \frac{ \sqrt{6}\|\vx\|\log m}{2\sqrt{md}} = \tilde{O}\left(\frac{1}{\sqrt{m}}\right),\quad\wp\geq 1-2re^{-\nicefrac{m}{32}} - 2e^{-\nicefrac{\log^2 m}{2}}.
\end{align*}

Therefore, the upper left block of $\displaystyle{\frac{\partial^2 g(\mc W;\x)}{\partial \mc W^2}\bigg|_{\Wc_\init}}$ will be close to $\zero$ with high probability, for large enough $m$.

Applying a similar analysis, we can  bound the lower right block, i.e., $\begin{bmatrix}\frac{\partial^2 g}{\left(\partial W_2^{(1)}\right)^2}  &\frac{\partial^2 g}{\partial W_2^{(1)} \partial W_2^{(2)}}  \\
   \frac{\partial^2 g}{\partial W_2^{(2)} \partial W_2^{(1)}}  & \frac{\partial^2 g}{\left(\partial W_2^{(2)}\right)^2} 
    \end{bmatrix}$ as well.  Specifically, we get that,
\begin{align*}
     \opnorm{\eval{\frac{\partial^2 g}{\partial W_2^{(1)} \partial W_2^{(2)}}}_{\mc{W} = \mc{W}_{\init}}} 
     \leq \frac{ \sqrt{6}\|\vx\|\log m}{2\sqrt{m}} = \tilde{O}\left(\frac{1}{\sqrt{m}}\right),\quad\wp\geq 1-2d e^{-\nicefrac{m}{32}} - 2re^{-\nicefrac{\log^2m}{2}}.
\end{align*}

Combining the results together,  with probability at least $1-4(d+r)e^{-\nicefrac{m}{32}} -4(r+e)2^{-\nicefrac{\log^2 m}{2}}$,
\begin{align*}
    \norm{\eval{\frac{\partial^2}{\partial \Wc^2} g(\mc{W})}_{\Wc_\init} - \begin{bmatrix}\zero &\H_\init \\
    \H_\init\tran & \zero
    \end{bmatrix}}_{\rm op} \leq \frac{2(\sqrt{6}+\sqrt{d})\|\vx\|\log m}{\sqrt{md}}= \tilde{O}\left(\frac{1}{\sqrt{m}}\right).
\end{align*}

Finally, we show $\opnorm{\H_{\init}}$ is separated from $0$.
Consider the unit vector with the same shape of flattened $\mc{W}$:
\begin{align*}
    \rvv = \frac{1}{\sqrt{\norm{W^{(1)}_{1,\init}\rve_1}^2+ \norm{W^{(1)}_{2,\init}\rve_2}^2}}\sbrac{\mathbf{0}\tran ,\underbrace{\rve_1^T{W^{(1)}_{1,\init}}\tran}_{\mathrm{w.r.t.}~\left(W_{1}^{(2)} \right)_{[1,:]}~\mathrm{in}~\mc{W}} ,\mathbf{0}\tran ,\underbrace{\rve_2^T{W^{(1)}_{2,\init}}\tran}_{\mathrm{w.r.t.}~W_{2}^{(2)} ~\mathrm{in}~\mc{W}} }\tran, 
\end{align*}
where $\rve_1 = (1,0,0,\cdots,0)^T \in\mathbb{R}^d$ and $\rve_2 = (1,0,0,\cdots,0)^T \in\mathbb{R}^r$.

Then we have
\begin{align*}
     \rvv\tran  \H_\init\rvv &= \frac{1}{m\sqrt{rd}}\frac{\rve_1^T {W_{1,\init}^{(1)}}\tran W_{1,\init}^{(1)}\vx \rve_2^T{W^{(1)}_{2,\init}}\tran {W^{(1)}_{2,\init}}\rve_2}{\norm{W^{(1)}_{1,\init}\rve_1}^2+ \norm{W^{(1)}_{2,\init}\rve_2}^2} = \frac{1}{m\sqrt{rd}}\frac{\|\vx\|\norm{ W_{1,\init}^{(1)}\rve_1}^2  \norm{{W^{(1)}_{2,\init}}\rve_2}^2}{\norm{W^{(1)}_{1,\init}\rve_1}^2+ \norm{W^{(1)}_{2,\init}\rve_2}^2}.
\end{align*}
Since $\norm{W^{(1)}_{1,\init}\rve_1}^2 \sim\chi^2(m)$ and $\norm{W^{(1)}_{2,\init}\rve_2}^2 \sim\chi^2(m)$, we can apply \Cref{lemma:chi_2} and use $t = 1/2$ to bound their norms. Specifically, we get with probability at least $1-4e^{-\nicefrac{m}{32}}$,
\begin{align*}
    \frac{m}{2}\leq \norm{W^{(1)}_{1,\init}\rve_1}^2\leq  \frac{3m}{2},~~~\frac{m}{2}\leq \norm{W^{(1)}_{2,\init}\rve_2}^2\leq  \frac{3m}{2}.
\end{align*}

Consequently, we have that,
\begin{align*}
\opnorm{\H_\init} \geq \left|\rvv\tran  \H_\init\rvv \right|\geq \frac{\norm{\x}}{24\sqrt{rd}}. \qquad\wp\geq1-4e^{-\nicefrac{m}{32}}.
\end{align*}
This concludes the proof of  \Cref{lem:bilinear_bounds}\ref{lem:bilinear_lower_bound}.
\end{proof}

\section{Neural networks with multiple bottlenecks}
\label{sec:multi}
In this section, we extend our results to deep linear networks with multiple bottleneck layers. We show that linear networks with $B-1$ bottlenecks will transition to $B$-degree polynomials of the weights, or more accurately, $B$-linear functions. Importantly, this is independent of the total depth of the network.

\begin{theorem}[Transition to $B$-th degree polynomials]\label{thm:deep_bnn}
Given a ball $\mathbb{B}(\mc{W}_{\init},R)$ with $R = O(1)$, with probability at least $1-e^{-\Omega(m)}$, $g(\mc{W})$ can be accurately approximated by a $B$-th degree polynomial. Specifically,
\begin{align}
    \norm{g(\mc{W}) - \left(J_{\mc{W}_{\init}}^B g(\cdot;\vx)\right)(\mc{W}-\mc{W}_{\init})}= O\left(\frac{1}{\sqrt{m}}\right),
\end{align}
where
\begin{align*}
   \left(J_{\mc{W}_{\init}}^B g(\cdot;\vx)\right)(\mc{W}-\mc{W}_{\init}) =  g(\mc{W}_\init;\vx) + \sum_{k=1}^B\frac{1}{k!} \inner{\eval{\frac{dg^k(\mc{W};\vx)}{d{\mc{W}}^k}}_{\mc{W}=\mc{W}_\init}, (\mc{W}-\mc{W}_\init)^{\otimes k}}.
\end{align*}
Furthermore, this polynomial function is multilinear in $\mc{W}$ with degree $B$. Consequently, as $m\rightarrow \infty$, $g(\mc{W};\vx)$ transitions to a $B$-th degree multilinear function of $\mc{W}$.
\end{theorem}

\begin{proof}[Proof of Theorem~\ref{thm:deep_bnn}] 
The proof follows a similar idea with the proof of Theorem~\ref{thm:bilinear}. Without lose of generality, we assume the output dimension of $g(\mc{W};\vx)$ is $1$, i.e., $d_B=1$. The result can be easily extend to multiple output of $g$ by analyzing each component of the output.

Given $R = O(1)$, by Taylor expansion with Lagrange remainder term, there exists a $\xivec \in \mathbb{B}(\mc{W}_{\init},R)$ such that
\begin{align*}
    g(\mc{W};\vx) - (J_{\mc{W}_\init}^B g(\cdot;\vx))\left(\mc{W}-\mc{W}_\init\right) = R_{B+1}(\mc{W},\xivec),
\end{align*}
 where
\begin{align*}
    R_{B+1}(\mc{W},\xivec) =  \frac{1}{(B+1)!} \inner{\eval{\frac{dg^{B+1}(\mc{W};x)}{d{\mc{W}}^{B+1}}}_{\mc{W}=\xivec}, (\mc{W}-\mc{W}_{\init})^{\otimes B+1}}.
\end{align*}
We use the following lemma which upper bounds the $(B+1)$-th derivative of $g$:
\begin{lemma}\label{lemma:upper_bound}
 With probability at least $1-2\sum_{b=1}^{B}(L_b-1)e^{-m/2} $ over random initialization $\mc{W}_0$, a BNN $g(\mathcal{W};\vx)$ with $B-1$ bottleneck layers  satisfies
\begin{align*}
     \opnorm{\frac{d^{B+1} g(\mathcal{W};\vx)}{d\mathcal{W}^{B+1}}} \leq \frac{(3+R/\sqrt{m})^{B}}{\sqrt{\prod_{b=1}^{B-1}d_b}}\frac{\|\vx\|}{\sqrt{m}}.
\end{align*}
for $\mc{W} \in \mathbb{B}(\mathcal{W}_{\init},R)$.
\end{lemma}

We apply Lemma~\ref{lemma:upper_bound}, and use Cauchy–Schwarz inequality:
\begin{align*}
    |R_{B+1}(\mc{W},\xivec)|&\leq \max_{\xivec \in \mathbb{B}(\mc{W}_{\init},R)}  \frac{1}{(B+1)!} \inner{\eval{\frac{dg^{B+1}(\mc{W};x)}{d{\mc{W}}^{B+1}}}_{\mc{W}=\xivec}, (\mc{W}-\mc{W}_{\init})^{\otimes B+1}} \\
    &\leq \frac{1}{(B+1)!}\max_{\xi \in \mathbb{B}(\mc{W}_{\init},R)} \opnorm{ \eval{\frac{dg^{B+1}(\mc{W};x)}{d{\mc{W}}^{B+1}}}_{\mc{W}=\xivec}} \norm{\mc{W}-\mc{W}_{\init}}^{B+1}\\
    &\leq \frac{(3+R/\sqrt{m})^{B}}{\sqrt{\prod_{b=1}^{B-1}d_b}}\frac{\|\vx\|}{\sqrt{m}}R^{B+1} = O\left(\frac{1}{\sqrt{m}}\right),
\end{align*}
with probability at least $1-2\sum_{b=1}^{B}(L_b-1)e^{-m/2} $.

We use the following lemma to lower bound the $B$-th derivative of $g$:
\begin{lemma}\label{lemma:lower_bound}
With probability at least $1-2\sum_{b=1}^B L_b e^{-m/32}$ over random initialization of $\mc{W}_\init$, a BNN $g(\mathcal{W};\vx)$ with $B-1$ bottleneck layers  satisfies
\begin{align*}
     \opnorm{\eval{\frac{d^{B} g(\mathcal{W};\vx)}{d\mathcal{W}^{B}}}_{\mc{W} =\mc{W}_{\init}}} \geq 
   2\|\vx\|\prod_{b=1}^B\frac{1}{2^{ L_b/2}\sqrt{ d_{b-1}}}.
\end{align*}
\end{lemma}

Therefore, by Lemma~\ref{lemma:lower_bound}, the degree of $(J_{\mc{W}_\init}^{B} g(\cdot;\vx))\left(\mc{W}-\mc{W}_\init\right)$ is guaranteed to be $B$. Note that if $\|\vx\|=0$, then $g$ is a zero function which is included in the class of $B$-th degree polynomial.
Picking $R = O(1)$, we finish proving the approximation result.

Since $g(\mc{W};x)$ is linear in $W_{b}^{(\ell)}$ for $b\in[B], \ell\in[L_b]$, according to Definition~\ref{def:multilinear}, $g(\mc{W};x)$ is a multilinear function. And the degree of multilinear will be $B$ when $m\rightarrow \infty$ as $g(\mc{W};x)$ will transition to a $B$-th degree polynomial. 
\end{proof}

\section{Conclusion and discussion}
In this work, we show a linear bottleneck neural network (BNN) with $B-1$ bottleneck layers transitions to a $B^{\rm th}$ degree polynomial of parameters when the width of non-bottleneck layers approaches to infinity. This expands the results about the transition to linearity of wide neural networks (WNN).

We also note that a similar result will hold if the first WNN block $f_1$, in a BNN has non-linear activations. 

\paragraph{Acknowledgements:} We are grateful for support of the NSF and the Simons Foundation for the Collaboration on the Theoretical Foundations of Deep Learning\footnote{\url{https://deepfoundations.ai/}} through awards DMS-2031883 and \#814639. We also acknowledge NSF support through  IIS-1815697 and the TILOS institute (NSF CCF-2112665).

\printbibliography

@article{vershynin2010introduction,
  title={Introduction to the non-asymptotic analysis of random matrices},
  author={Vershynin, Roman},
  journal={arXiv preprint arXiv:1011.3027},
  year={2010}
}

@inproceedings{jacot2018neural,
  title={Neural tangent kernel: Convergence and generalization in neural networks},
  author={Jacot, Arthur and Gabriel, Franck and Hongler, Cl{\'e}ment},
  booktitle={Advances in neural information processing systems},
  pages={8571--8580},
  year={2018}
}

@inproceedings{du2018gradientdeep,
  title={Gradient Descent Finds Global Minima of Deep Neural Networks},
  author={Du, Simon and Lee, Jason and Li, Haochuan and Wang, Liwei and Zhai, Xiyu},
  booktitle={International Conference on Machine Learning},
  pages={1675--1685},
  year={2019}
}

@inproceedings{arora2019fine,
  title={Fine-Grained Analysis of Optimization and Generalization for Overparameterized Two-Layer Neural Networks},
  author={Arora, Sanjeev and Du, Simon and Hu, Wei and Li, Zhiyuan and Wang, Ruosong},
  booktitle={International Conference on Machine Learning},
  pages={322--332},
  year={2019}
}

@inproceedings{allen2019convergence,
  title={A Convergence Theory for Deep Learning via Over-Parameterization},
  author={Allen-Zhu, Zeyuan and Li, Yuanzhi and Song, Zhao},
  booktitle={International Conference on Machine Learning},
  pages={242--252},
  year={2019}
}

@inproceedings{zou2019improved,
  title={An improved analysis of training over-parameterized deep neural networks},
  author={Zou, Difan and Gu, Quanquan},
  booktitle={Advances in Neural Information Processing Systems},
  pages={2053--2062},
  year={2019}
}

@article{liu2020linearity,
  title={On the linearity of large non-linear models: when and why the tangent kernel is constant},
  author={Liu, Chaoyue and Zhu, Libin and Belkin, Mikhail},
  journal={Advances in Neural Information Processing Systems},
  volume={33},
  year={2020}
}

@inproceedings{he2016deep,
  title={Deep residual learning for image recognition},
  author={He, Kaiming and Zhang, Xiangyu and Ren, Shaoqing and Sun, Jian},
  booktitle={Proceedings of the IEEE conference on computer vision and pattern recognition},
  pages={770--778},
  year={2016}
}

@article{liu2020loss,
  title={Loss landscapes and optimization in over-parameterized non-linear systems and neural networks},
  author={Liu, Chaoyue and Zhu, Libin and Belkin, Mikhail},
  journal={Applied and Computational Harmonic Analysis},
  year={2022},
  publisher={Elsevier}
}

@article{chizat2019lazy,
  title={On lazy training in differentiable programming},
  author={Chizat, Lenaic and Oyallon, Edouard and Bach, Francis},
  journal={Advances in Neural Information Processing Systems},
  volume={32},
  year={2019}
}

@inproceedings{yu2011improved,
  title={Improved bottleneck features using pretrained deep neural networks},
  author={Yu, Dong and Seltzer, Michael L},
  booktitle={Twelfth annual conference of the international speech communication association},
  year={2011}
}

@inproceedings{matejka2014neural,
  title={Neural Network Bottleneck Features for Language Identification.},
  author={Matejka, Pavel and Zhang, Le and Ng, Tim and Glembek, Ondrej and Ma, Jeff Z and Zhang, Bing and Mallidi, Sri Harish},
  booktitle={Odyssey},
  year={2014}
}

@article{lecun2015deep,
  title={Deep learning},
  author={LeCun, Yann and Bengio, Yoshua and Hinton, Geoffrey},
  journal={nature},
  volume={521},
  number={7553},
  pages={436--444},
  year={2015},
  publisher={Nature Publishing Group}
}

@article{cao2019generalization,
  title={Generalization bounds of stochastic gradient descent for wide and deep neural networks},
  author={Cao, Yuan and Gu, Quanquan},
  journal={Advances in neural information processing systems},
  volume={32},
  year={2019}
}

@article{moroshko2020implicit,
  title={Implicit bias in deep linear classification: Initialization scale vs training accuracy},
  author={Moroshko, Edward and Woodworth, Blake E and Gunasekar, Suriya and Lee, Jason D and Srebro, Nati and Soudry, Daniel},
  journal={Advances in neural information processing systems},
  volume={33},
  pages={22182--22193},
  year={2020}
}

@article{radhakrishnan2020alignment,
  title={On Alignment in Deep Linear Neural Networks},
  author={Radhakrishnan, Adityanarayanan and Nichani, Eshaan and Bernstein, Daniel and Uhler, Caroline},
  journal={arXiv preprint arXiv:2003.06340},
  year={2020}
}

@inproceedings{ji2018gradient,
  title={Gradient descent aligns the layers of deep linear networks},
  author={Ji, Ziwei and Telgarsky, Matus},
  booktitle={International Conference on Learning Representations},
  year={2018}
}

@inproceedings{arora2018convergence,
  title={A Convergence Analysis of Gradient Descent for Deep Linear Neural Networks},
  author={Arora, Sanjeev and Cohen, Nadav and Golowich, Noah and Hu, Wei},
  booktitle={International Conference on Learning Representations},
  year={2018}
}

@inproceedings{bartlett2018gradient,
  title={Gradient descent with identity initialization efficiently learns positive definite linear transformations by deep residual networks},
  author={Bartlett, Peter and Helmbold, Dave and Long, Philip},
  booktitle={International conference on machine learning},
  pages={521--530},
  year={2018},
  organization={PMLR}
}

@inproceedings{du2019width,
  title={Width provably matters in optimization for deep linear neural networks},
  author={Du, Simon and Hu, Wei},
  booktitle={International Conference on Machine Learning},
  pages={1655--1664},
  year={2019},
  organization={PMLR}
}

@article{arora2019implicit,
  title={Implicit regularization in deep matrix factorization},
  author={Arora, Sanjeev and Cohen, Nadav and Hu, Wei and Luo, Yuping},
  journal={Advances in Neural Information Processing Systems},
  volume={32},
  year={2019}
}

@inproceedings{tarmoun2021understanding,
  title={Understanding the dynamics of gradient flow in overparameterized linear models},
  author={Tarmoun, Salma and Franca, Guilherme and Haeffele, Benjamin D and Vidal, Rene},
  booktitle={International Conference on Machine Learning},
  pages={10153--10161},
  year={2021},
  organization={PMLR}
}

@article{netrapalli2013phase,
  title={Phase retrieval using alternating minimization},
  author={Netrapalli, Praneeth and Jain, Prateek and Sanghavi, Sujay},
  journal={Advances in Neural Information Processing Systems},
  volume={26},
  year={2013}
}

@article{rigollet2015high,
  title={High dimensional statistics},
  author={Rigollet, Phillippe and H{\"u}tter, Jan-Christian},
  journal={Lecture notes for course 18S997},
  volume={813},
  number={814},
  pages={46},
  year={2015}
}

@article{gidel2019implicit,
  title={Implicit regularization of discrete gradient dynamics in linear neural networks},
  author={Gidel, Gauthier and Bach, Francis and Lacoste-Julien, Simon},
  journal={Advances in Neural Information Processing Systems},
  volume={32},
  year={2019}
}

@article{gunasekar2018implicit,
  title={Implicit bias of gradient descent on linear convolutional networks},
  author={Gunasekar, Suriya and Lee, Jason D and Soudry, Daniel and Srebro, Nati},
  journal={Advances in Neural Information Processing Systems},
  volume={31},
  year={2018}
}

@inproceedings{jain2013low,
  title={Low-rank matrix completion using alternating minimization},
  author={Jain, Prateek and Netrapalli, Praneeth and Sanghavi, Sujay},
  booktitle={Proceedings of the forty-fifth annual ACM symposium on Theory of computing},
  pages={665--674},
  year={2013}
}

\newpage
\noindent {\Huge \textbf{Appendix}}
\appendix

\section{Proof of Lemma~\ref{lemma:upper_bound} }

Since $\mc{W} = (\rmW_1,\cdots,\rmW_B)$ and for each $b\in[B]$, $\rmW_b = (W_b^{(1)},\cdots,W_b^{(L_b)})$,  we split the tensor $\frac{d^{B+1} g(\mathcal{W};\vx)}{d\mathcal{W}^{B+1}}$ into blocks where each block is with respect to weight matrices $W_b^{(\ell)}$ for $b\in[B],~\ell\in[L_b]$. Specifically, each block will take the following form:
\begin{align*}
    \frac{d^{B+1} g(\mc{W};\vx)}{\partial W_{i_1}^{(j_1)}\cdots \partial W_{i_{B+1}}^{(j_{B+1})}},~~~~k\in[B+1],~i_k\in[B],~j_k\in[L_{i_k}].
\end{align*}

Define $C_{L,B}=\left(\sum_{b=1}^B L_b\right)^{B+1}$.
Therefore, there will be in total $C_{L,B}^{B+1}$ blocks. By Lemma~\ref{lemma:tensor_norm}, the spectral norm of $\frac{d^{B+1} g(\mathcal{W};\vx)}{d\mathcal{W}^{B+1}}$ is bounded by the summation of the spectral norm of tensor blocks. Due to the symmetry of the derivative, we only need to consider the ``upper triangle'' of the tensor:
\begin{align*}
    \opnorm{\frac{d^{B+1} g(\mc{W};\vx)}{d\mc{W}^{B+1}}} \leq \frac{C_{L,B}^{B+1}}{\binom{C_{L,B}}{B}}\sum_{\sum_{b=1}^{i_1-1}L_b+j_1 \leq \sum_{b=1}^{i_2-1}L_b+j_2\leq \cdots \leq \sum_{b=1}^{i_{B+1}-1}L_b+j_{B+1}  }\opnorm{    \frac{d^{B+1} g(\mc{W};\vx)}{\partial W_{i_1}^{(j_1)}\cdots \partial W_{i_{B+1}}^{(j_{B+1})}}}.
\end{align*}

By definition, given the tensor block, letting $\rmV_{B+1}= (V_1,\cdots V_{B+1})$ where $V_b$ has the same shape with $W_{i_b}^{j_b}$ for $b\in[B+1]$,  we have
\begin{align*}
   \opnorm{  \frac{\partial^{B+1} g(\mc{W};\vx)}{\partial W_{i_1}^{(j_1)}\cdots \partial W_{i_{B+1}}^{(j_{B+1})}}} = \sup_{\|V_b\|_F=1, \forall b\in[B+1]} \innert{   \frac{\partial^{B+1} g(\mc{W};\vx)}{\partial W_{i_1}^{(j_1)}\cdots \partial W_{i_{B+1}}^{(j_{B+1})}}, \rmV_{B+1}}.
\end{align*}

Note that $g(\mc{W};\vx)$ is linear in each $W_{i_b}^{j_b}$. Therefore, $  \innert{\frac{\partial^{B+1} g(\mc{W};\vx)}{\partial W_{i_1}^{(j_1)}\cdots \partial W_{i_{B+1}}^{(j_{B+1})}}, \rmV_{B+1}}$ is equal to the output of $g(\mc{W};\vx)$ by replacing $W_{i_b}^{j_b}$ in $\mc{W}$ by $V_b$ for $b\in[B+1]$. 

Specifically, since each $g_b(\mc{W}_b;\vx) = f_b(W_b;g_{b-1}(\mc{W}_{b-1};\vx))$ is linear in $g_{b-1}(\mc{W}_{b-1};\vx)$, we can equivalently write  $  \inner{\frac{\partial^{B+1} g(\mc{W};\vx)}{\partial W_{i_1}^{(j_1)}\cdots \partial W_{i_{B+1}}^{(j_{B+1})}}, \rmV_{B+1}}$ in a recursive format by distributing $\rmV_{B+1}$ to $B$ WNNs:
\begin{align*}
        h_1(\mc{W}_1;\vx) &= \innert{ \frac{\partial^{p_1} f_1(\rmW_1;\vx)}{\partial W_1^{(a_{\ell_1})}\cdots \partial W_1^{(a_{\ell{p_1}})}}, \rmV_{p_1}},\\
    h_b(\mathcal{W}_b;\vx) &= \innert{\frac{\partial^{p_b} f_b(\rmW_b; h_{b-1}(\mathcal{W}_{b-1};\vx))}{\partial W_b^{(a_{\ell_b})}\cdots \partial W_b^{(a_{\ell{p_b}})}}, \rmV_{p_b}},~~~2\leq b\leq B,\\
    h(\mc{W};\vx) &= h_B(\mc{W}_B;\vx) =  \innert{\frac{\partial^{B+1} g(\mc{W};\vx)}{\partial W_{i_1}^{(j_1)}\cdots \partial W_{i_{B+1}}^{(j_{B+1})}}, \rmV_{B+1}},
\end{align*}

where $\sum_{b=1}^B {p_b} = B+1$ and $\bigcup_{b=1}^B\{(b,a_{\ell_p})\}_{p=1}^{p_b} = \{(i_1,j_1),\cdots, (i_{B+1},j_{B+1})\}$. And $\rmV_{p_b} = (V_{\ell_1},\cdots,V_{\ell_{p_b}})$ where each $V_{\ell_i}$ has same shape with $W_b^{(a_{i})}$ with $\norm{V_{\ell_i}}_F = 1$ for $b\in[B]$, $i\in[p_b]$ and $(\rmV_{p_1},\cdots \rmV_{p_B}) = \rmV_{B+1}$.

Note that 

\begin{align}\label{eq:h_b_recur}
    \norm{ h_b(\mathcal{W}_b;\vx)} \leq \sup_{\|\rvz\| = \norm{h_{b-1}(\mathcal{W}_{b-1};\vx)}}\opnorm{\frac{\partial^{p_b} f_b(\rmW_b; \rvz)}{\partial W_b^{(a_{\ell_b})}\cdots \partial W_b^{(a_{\ell{p_b}})}}}.
\end{align}

If $p_b = 0$,  according to Lemma~\ref{lemma:outout_wnn}, the RHS is bounded by
\begin{align}\label{eq:h_b_0}
\sup_{\|\rvz\| = \norm{h_{b-1}(\mathcal{W}_{b-1};\vx)}}\opnorm{\frac{\partial^{p_b} f_b(\rmW_b; \rvz)}{\partial W_b^{(a_{\ell_b})}\cdots \partial W_b^{(a_{\ell{p_b}})}}} &=\sup_{\|\rvz\|= \norm{h_{b-1}(\mathcal{W}_{b-1};\vx)}}\norm{f_b(\rmW_b;\rvz)}\nonumber\\
&\leq  \sqrt{d_b}(\sqrt{6}/2 +R/\sqrt{m})^{L_b} \log m \frac{\norm{h_{b-1}(\mathcal{W}_{b-1};\vx)}}{\sqrt{d_{b-1}}}
\end{align}
with probability at least $1-2d_b e^{-\nicefrac{\log m}{2}}-2d_b(L_b-1)\exp(-m/32)$.

If $p_b>0$, by Lemma~\ref{lemma:p_derivative},
\begin{align}\label{eq:h_b_1}
\sup_{\|\rvz\| = \norm{h_{b-1}(\mathcal{W}_{b-1};\vx)}}\opnorm{\frac{\partial^{p_b} f_b(\rmW_b; \rvz)}{\partial W_b^{(a_{\ell_b})}\cdots \partial W_b^{(a_{\ell{p_b}})}}} \leq \frac{(3\sqrt{m}+R)^{L_b-p_b}}{(\sqrt{m})^{L_b-1} \sqrt{d_{b-1}}}\norm{h_{b-1}(\mathcal{W}_{b-1};\vx)},
\end{align}
with probability at least $1-2(L_b-p_b)\exp(-m/2)$.

Comparing these two results, i.e., Eq.~(\ref{eq:h_b_0}) and~(\ref{eq:h_b_1}), we notice that if $p_b = 0$, the RHS of Eq.~(\ref{eq:h_b_recur}) is bounded by $\tilde{O}(1)$, otherwise it is bounded by $O((\sqrt{m})^{1-p_b})$. Therefore, if $p_b = 1$, two bounds are of the same order; if $p_b>1$,  Eq.~(\ref{eq:h_b_1}) provides a better bound.

Note Eq.~(\ref{eq:h_b_0}) and~(\ref{eq:h_b_1}) are in a recursive format, and ultimately lead to a bound on $\norm{h_1(\mc{W};\vx)}$. Since $\sum_{b=1}^B {p_b} = B+1$, the worst case, i.e., the largest bound on $\norm{h(\mc{W};\vx)}$ will be first letting $p_1=p_2=\cdots=p_b = 1$, then adding $1$ to any $p_i$ since the bound decreases by the same factor $1/\sqrt{m}$. Without lose of generality, we pick $p_B=2$. As a result, with probability at least $1-2\sum_{b=1}^{B-1}(L_b-1)\exp(-m/2) - 2(L_B-2)\exp(-m/2)$,
\begin{align*}
    \norm{h(\mc{W};\vx)} \leq \frac{(3+R/\sqrt{m})^{B}}{\sqrt{\prod_{b=1}^{B-1}d_b}}\frac{\|\vx\|}{\sqrt{m}}.
\end{align*}

This concludes the proof of \Cref{lemma:upper_bound}.

\section{Proof of Lemma~\ref{lemma:lower_bound}}
We will find a direction $\rmV_B = (\rvv_1,\cdots,\rvv_B)$ where each $\rvv_b$ has the same shape with $\Vec(\mc{W})$ and $\|\rvv_b\| = 1$ for all $b\in[B]$ such that  
\begin{align*}
    \innert{ \frac{d^{B} g(\mc{W};\vx)}{d\mc{W}^{B}},\rmV_B}= \Omega(1).
\end{align*}

Therefore, 
\begin{align*}
  \opnorm{\frac{d^{B} g(\mc{W};\vx)}{d\mc{W}^{B}}} \geq \innert{ \frac{d^{B} g(\mc{W};\vx)}{d\mc{W}^{B}},\rmV_B } = \Omega(1).
\end{align*}

Note that the weight matrices at bottleneck layers of $g$ have rank smaller than $m$, e.g., $d_b$. We take the derivative of $g$ with respect to those matrices, and find a direction vector which aligns with the derivative. This low rank property of the weight matrices at bottleneck layers will give us a lower bound for the norm of the $B$-th derivative of $g$.

Specifically, for each $b\in[B]$, we choose 
\begin{align*}
    \rvv_b =\left(\mathbf{0}^T,\cdots, \mathbf{0}^T, \underbrace{\frac{\rvu_{b,\init}^T}{\norm{\rvu_{b,\init}}}}_{\mathrm{w.r.t.}~\left(W_{b}^{(L_b)} \right)_{[1,:]}~\mathrm{in}~\mc{W}},\cdots, \mathbf{0}^T\right)^T,
\end{align*}

where $\rve_b^1 = (1,0,0,\cdots,0) \in\mathbb{R}^{d_{b-1}}$, and
\begin{align*}
\rvu_{b,\init} := {\frac{1}{\sqrt{m}}W_{b,\init}^{(L_b-1)}\cdots \frac{1}{\sqrt{d_{b-1}}}W_{b,\init}^{(1)}\rve_b^1} \in \mathbb{R}^m.
\end{align*}

It is not hard to see that 
\begin{align*}
    \innert{ \frac{d^{B} g(\mc{W};\vx)}{d\mc{W}^{B}},\rmV_B } = 2 \|\vx\|(\sqrt{m})^{-B}\prod_{b=1}^{B}\|\rvu_{b,\init}\|.
\end{align*}

By Lemma~\ref{lemma:lowerbound_b}, we can bound each $\|\rvu_{b,\init}\|$ for $b\in[B]$. Then we apply union bound over indices $[B]$. We have with probability at least $1-2\sum_{b=1}^B L_b e^{-m/32}$,
\begin{align*}
  \innert{ \frac{d^{B} g(\mc{W};\vx)}{d\mc{W}^{B}},\rmV_B }  \geq \frac{2\|\vx\|}{2^{\sum_{b=1}^B L_b/2}\sqrt{\prod_{b=1}^B d_{b-1}}}.
\end{align*}
This concludes the proof for \Cref{lemma:lower_bound}.

\section{Lemmas for deep bottleneck networks}

\begin{lemma}\label{lemma:lowerbound_b}
For each $b\in[B]$, with probability at least $1-2L_be^{-m/32}$,
\begin{align*}
    \|\rvu_{b,\init}\| \geq \frac{\sqrt{m}}{2^{L_b/2}\sqrt{d_{b-1}}}.
\end{align*}
\end{lemma}
\begin{proof}
Due to the recursive format of $\rvu_b$, we will prove the result recursively.

When $k=1$, since $\|W_{b,\init}^{(1)}\rve_b^1\|^2 \sim \chi^2(m)$, we use the tail bound for $\chi^2$ distribution. We pick $t= 1/2$ in Lemma~\ref{lemma:chi_2},  hence we have  with probability at least $1-2e^{-m/32}$,
\begin{align*}
    \|W_{b,\init}^{(1)}\rve_b^1\|^2 \geq m/2,
\end{align*}

hence $\frac{1}{\sqrt{d_{b-1}}} \|W_{b,\init}^{(1)}\rve_b^1\| \geq\frac{\sqrt{m}}{\sqrt{2d_{b-1}}}$.

Suppose when $k=\ell$, with probability at least $1-2\ell e^{-m/32}$,
\begin{align*}
   \left\|\frac{1}{\sqrt{m}} W_{b,\init}^{(\ell)} \frac{1}{\sqrt{m}}W_{b,\init}^{(\ell-1)}\cdots \frac{1}{\sqrt{d_{b-1}}}W_{b,\init}^{(1)}\rve_b^1\right\| \geq \frac{\sqrt{m}}{2^{\ell/2}\sqrt{d_{b-1}}}.
\end{align*}
Then when $k=\ell+1$, we have
\begin{align*}
     \frac{\norm{W_{b,\init}^{(\ell+1)}\frac{1}{\sqrt{m}}W_{b,\init}^{(\ell)} \frac{1}{\sqrt{m}}W_{b,\init}^{(\ell-1)}\cdots \frac{1}{\sqrt{d_{b-1}}}W_{b,\init}^{(1)}\rve_b^1}^2}{\left\|\frac{1}{\sqrt{m}}W_{b,\init}^{(\ell)} \frac{1}{\sqrt{m}}W_{b,\init}^{(\ell-1)}\cdots \frac{1}{\sqrt{d_{b-1}}}W_{b,\init}^{(1)}\rve_b^1\right\|^2} \sim \chi^2(m),
\end{align*}
with probability at least $1-2\ell e^{-m/32}$ over the randomness of $W_{b,\init}^{(1)},\cdots W_{b,\init}^{(l)}$.

Apply Lemma~\ref{lemma:chi_2} again and we pick $t=1/2$, we have 
\begin{align*}
      \norm{W_{b,\init}^{(\ell+1)}\frac{1}{\sqrt{m}}W_{b,\init}^{(\ell)} \frac{1}{\sqrt{m}}W_{b,\init}^{(\ell-1)}\cdots \frac{1}{\sqrt{d_{b-1}}}W_{b,\init}^{(1)}\rve_b^1}^2 \geq \frac{m}{2}\left\|\frac{1}{\sqrt{m}}W_{b,\init}^{(\ell)} \frac{1}{\sqrt{m}}W_{b,\init}^{(\ell-1)}\cdots \frac{1}{\sqrt{d_{b-1}}}W_{b,\init}^{(1)}\rve_b^1\right\|^2.
\end{align*}
with probability at least $1-2e^{-m/32}$. By the inductive assumption and applying union bound, we have with probability at least $1-2(\ell+1)e^{-m/32}$,
\begin{align*}
     \norm{\frac{1}{\sqrt{m}}W_{b,\init}^{(\ell+1)}\frac{1}{\sqrt{m}}W_{b,\init}^{(\ell)} \frac{1}{\sqrt{m}}W_{b,\init}^{(\ell-1)}\cdots \frac{1}{\sqrt{d_{b-1}}}W_{b,\init}^{(1)}\rve_b^1}  \leq \frac{\sqrt{m}}{2^{(\ell+1)/2}\sqrt{d_{b-1}}},
\end{align*}
which completes the recursive step hence finishes the proof.
    
\end{proof}

\begin{lemma}\label{lemma:outout_wnn}
Given an $L$-layer WNN $f(\rmW;\vx) \in\mathbb{R}^c$, with probability at least $1-2ce^{-\nicefrac{\log m}{2}}-2c(L-1)\exp(-m/32)$
\begin{align*}
    \|f(\rmW;\vx)\| \leq (\sqrt{6}/2 +R/\sqrt{m})^{L}\log m\frac{\sqrt{c}\|\vx\|}{\sqrt{d}},
\end{align*}
in the ball $\mathbb{B}(\rmW_{\init},R)$.
\end{lemma}

\begin{proof}

We will first prove the norm of the neurons in the last hidden layers satisfies
\begin{align}\label{eq:hidden_neurons}
     \norm{\frac{1}{\sqrt{m}}W^{(L-1)}\cdots \frac{1}{\sqrt{d}}W^{(1)}\vx} \leq (\sqrt{6}/2 +R/\sqrt{m})^{L-1} \|\vx\|
\end{align}
with probability at least $1-2(L-1)\exp(-m/32)$.

We will bound the result recursively. When $k=1$, we have
\begin{align*}
   \norm{ \frac{1}{\sqrt{d}}W^{(1)}\vx} \leq   \norm{ \frac{1}{\sqrt{d}}W_{\init}^{(1)}\vx} + \norm{ \frac{1}{\sqrt{d}}\left(W^{(1)} - W_{\init}^{(1)}\right)\vx}.
\end{align*}

Note that $ \norm{W_{\init}^{(1)}\vx}^2/\|\vx\|^2\sim \chi^2(m)$. Using Lemma~\ref{lemma:chi_2} and picking $t=1/2$, we have with probability at least $1-2e^{-\nicefrac{m}{32}}$.,
\begin{align*}
    \norm{W_{\init}^{(1)}\vx}^2/\|\vx\|^2 \leq \frac{3m}{2}.
\end{align*}

Then we have
\begin{align*}
      \norm{ \frac{1}{\sqrt{d}}W^{(1)}\vx}  \leq \frac{\sqrt{6}\sqrt{m}\|\vx\|}{2\sqrt{d}} + \frac{R\|\vx\|}{\sqrt{d}} = \frac{\sqrt{m}\|\vx\|}{\sqrt{d}}(\sqrt{6}/2 + R/\sqrt{m}),
\end{align*}
with probability at least  $1-2e^{-\nicefrac{m}{32}}$.

Suppose when $k=\ell$, with probability at least $1-2\ell\exp(-m/32)$
\begin{align*}
    \norm{\frac{1}{\sqrt{m}}W^{(\ell)}\cdots \frac{1}{\sqrt{d}}W^{(1)}\vx} \leq  \frac{\sqrt{m}\|\vx\|}{\sqrt{d}}(\sqrt{6}/2 +R/\sqrt{m})^\ell .
\end{align*}

When $k=\ell+1$, we have
\begin{align*}
    \frac{1}{\sqrt{m}}W^{(\ell+1)}\cdots \frac{1}{\sqrt{d}}W^{(1)}\vx = \frac{1}{\sqrt{m}}W_{\init}^{(\ell+1)}\cdots \frac{1}{\sqrt{d}}W^{(1)}\vx + \frac{1}{\sqrt{m}}(W^{(\ell+1)} - W_{\init}^{(\ell+1)})\cdots \frac{1}{\sqrt{d}}W^{(1)}\vx.
\end{align*}
Note that
\begin{align*}
    \frac{\norm{\frac{1}{\sqrt{m}}W_{\init}^{(\ell+1)}\cdots \frac{1}{\sqrt{d}}W^{(1)}\vx}^2}{\norm{ \frac{1}{\sqrt{m}}W^{(\ell)}\frac{1}{\sqrt{d}}W^{(1)}\vx}^2} \sim\chi^2(m),
\end{align*}
with respect to $W_{\init}^{(\ell+1)}$.

Using Lemma~\ref{lemma:chi_2} and picking $t=1/2$, we have with probability at least $1-2\exp(-m/32)$,
\begin{align*}
    \norm{\frac{1}{\sqrt{m}}W_{\init}^{(\ell+1)}\cdots \frac{1}{\sqrt{d}}W^{(1)}\vx}^2 \leq \frac{3}{2}m\norm{ \frac{1}{\sqrt{m}}W^{(\ell)}\frac{1}{\sqrt{d}}W^{(1)}\vx}^2.
\end{align*}

Since $\norm{W^{(\ell+1)} - W_{\init}^{(\ell+1)}}\leq R$, we have with probability at least $1-2\exp(-m/32)$,
\begin{align*}
    \norm{\frac{1}{\sqrt{m}}W^{(\ell+1)}\cdots \frac{1}{\sqrt{d}}W^{(1)}\vx} &\leq \norm{ \frac{1}{\sqrt{m}}W_{\init}^{(\ell+1)}\cdots \frac{1}{\sqrt{d}}W^{(1)}\vx} + \norm{\frac{1}{\sqrt{m}}(W^{(\ell+1)} - W_{\init}^{(\ell+1)})\cdots \frac{1}{\sqrt{d}}W^{(1)}\vx}\\
    &\leq (\sqrt{6}/2+R/\sqrt{m})\norm{\frac{1}{\sqrt{m}}W^{(\ell)}\cdots \frac{1}{\sqrt{d}}W^{(1)}\vx}.
\end{align*}
By the inductive assumption, we apply union bound and we have with probability at least $1-2(\ell+1)\exp(-m/32)$,
\begin{align*}
     \norm{\frac{1}{\sqrt{m}}W^{(\ell+1)}\cdots \frac{1}{\sqrt{d}}W^{(1)}\vx}\leq (\sqrt{6}/2+\frac{R}{\sqrt{m}})\norm{\frac{1}{\sqrt{m}}W^{(\ell)}\cdots \frac{1}{\sqrt{d}}W^{(1)}\vx} \leq  \frac{\sqrt{m}\|\vx\|}{\sqrt{d}}(\sqrt{6}/2+\frac{R}{\sqrt{m}})^{\ell+1}.
\end{align*}
Then we finish the inductive step hence finish the proof.

Now we prove the result in the lemma. For each output $f_k$, $k\in[c]$,
\begin{align*}
      f_k(\rmW;\vx) &= \frac{1}{\sqrt{m}} W_{[k,:]}^{(L)}\frac{1}{\sqrt{m}}W^{(L-1)}\cdots \frac{1}{\sqrt{d}}W^{(1)}\vx \\
      &= \frac{1}{\sqrt{m}} W_{[k,:],\init}^{(L)}\frac{1}{\sqrt{m}}W^{(L-1)}\cdots \frac{1}{\sqrt{d}}W^{(1)}\vx + \frac{1}{\sqrt{m}} \left(W_{[k,:]}^{(L)}-W_{[k,:],\init}^{(L)}\right)\frac{1}{\sqrt{m}}W^{(L-1)}\cdots \frac{1}{\sqrt{d}}W^{(1)}\vx.
\end{align*}
Note that $\frac{1}{\sqrt{m}} W_{[k,:],\init}^{(L)}\frac{1}{\sqrt{m}}W^{(L-1)}\cdots \frac{1}{\sqrt{d}}W^{(1)}\vx \sim \mathcal{N}\left(0,\norm{\frac{1}{\sqrt{m}}W^{(L-1)}\cdots \frac{1}{\sqrt{d}}W^{(1)}\vx}^2/m\right)$ with respect to $W_{[k,:],\init}^{(L)}$. By concentration equality for Gaussian random variables, i.e., Lemma~\ref{lemma:gaussian}, picking $t= \norm{\frac{1}{\sqrt{m}}W^{(L-1)}\cdots \frac{1}{\sqrt{d}}W^{(1)}\vx}\frac{\log m}{\sqrt{m}}$,
with probability at least $1-2e^{-\log^2m/2}$ over the randomness of $W_{[k,:],\init}^{(L)}$,
\begin{align*}
    \left|\frac{1}{\sqrt{m}} W_{[k,:],\init}^{(L)}\frac{1}{\sqrt{m}}W^{(L-1)}\cdots \frac{1}{\sqrt{d}}W^{(1)}\vx \right|\leq \norm{\frac{1}{\sqrt{m}}W^{(L-1)}\cdots \frac{1}{\sqrt{d}}W^{(1)}\vx}\frac{\log m}{\sqrt{m}}.
\end{align*}
Then we use Eq.~(\ref{eq:hidden_neurons}) to bound the RHS of the above equation. By union bound, with probability at least $1-2e^{-\log^2 m/2}-2(L-1)\exp(-m/32)$,
\begin{align*}
     \left|\frac{1}{\sqrt{m}} W_{[k,:],\init}^{(L)}\frac{1}{\sqrt{m}}W^{(L-1)}\cdots \frac{1}{\sqrt{d}}W^{(1)}\vx \right| \leq (\sqrt{6}/2 +R/\sqrt{m})^{L-1}\log m \frac{\|\vx\|}{\sqrt{d}}.
\end{align*}

Since $\norm{W_{[k,:]}^{(L)}-W_{[k,:],\init}^{(L)}} \leq R$ in the ball, we have 
\begin{align*}
   \left| \frac{1}{\sqrt{m}} \left(W_{[k,:]}^{(L)}-W_{[k,:],\init}^{(L)}\right)\frac{1}{\sqrt{m}}W^{(L-1)}\cdots \frac{1}{\sqrt{d}}W^{(1)}\vx\right|& \leq R/\sqrt{m} \norm{\frac{1}{\sqrt{m}}W^{(L-1)}\cdots \frac{1}{\sqrt{d}}W^{(1)}\vx}\\
   &\leq R(\sqrt{6}/2 +R/\sqrt{m})^{L-1}\log m \frac{\|\vx\|}{\sqrt{d}}.
\end{align*}

As a result, with probability at least $1-2e^{-\nicefrac{\log m}{2}}-2(L-1)\exp(-m/32)$,
\begin{align*}
    |f_k(\rmW;\vx)| \leq (1+R/\sqrt{m})(\sqrt{6}/2 +R/\sqrt{m})^{L-1}\log m \frac{\|\vx\|}{\sqrt{d}} \leq (\sqrt{6}/2 +R/\sqrt{m})^{L}\log m\frac{\|\vx\|}{\sqrt{d}}.
\end{align*}

Use union bound over the indices $k\in[c]$, we have with probability at least $1-2ce^{-\nicefrac{\log m}{2}}-2c(L-1)\exp(-m/32)$,
\begin{align*}
    \norm{f(\rmW;\vx)}\leq (\sqrt{6}/2 +R/\sqrt{m})^{L}\log m\frac{\sqrt{c}\|\vx\|}{\sqrt{d}}.
\end{align*}
\end{proof}

\begin{lemma}\label{lemma:p_derivative}
 For an $L$-layer WNN $f(\rmW;\vx)$ where $\rmW = (W^{(1)},\cdots W^{(L)})$, given $0<p\leq L$, then for $a_i \in [L], i\in[p]$,  with probability at least $1-2(L-p)\exp(-m/2)$,
\begin{align*}
     \opnorm{\frac{\partial^{p} f(\rmW;\vx)}{\partial W^{(a_1)}\cdots \partial W^{(a_{p})}}} \leq \frac{(3\sqrt{m}+R)^{L-p}\|\vx\|}{(\sqrt{m})^{L-1} \sqrt{d}}.
\end{align*}
in the ball $\mathbb{B}(\rmW_{\init},R)$.
\end{lemma}

\begin{proof}
Since $f(\rmW;\vx)$ is linear in each $W^{(\ell)}$, $\ell\in[L]$, it is not hard too see that if $a_i=a_j$ for some $i,j\in[p], i\neq j$, we have $\opnorm{\frac{\partial^{p} f(\rmW;\vx)}{\partial W^{(a_1)}\cdots \partial W^{(a_{p})}}}= 0$.

We consider $p$ distinct indices $\{a_1,\cdots,a_p\}$. We denote $\rmV_p = \left(V_1,\cdots,V_p\right)$ where $V_\ell$ has the same shape with $W^{(a_\ell)}$, $\ell\in[p]$. 
Then by definition, the norm takes the form
\begin{align*}
    \opnorm{\frac{\partial^{p} f(\rmW;\vx)}{\partial W^{(a_1)}\cdots \partial W^{(a_{p})}}}= \sup_{\|V_1\|_F = \cdots =\|V_p\|_F = 1}\innert{\frac{\partial^{p} f(\rmW;\vx)}{\partial W^{(a_1)}\cdots \partial W^{(a_{p})}}, \rmV_p}
\end{align*}

Again, since  $f(\rmW;\vx)$ is linear in each $W^{(\ell)}$, $\ell\in[L]$, the inner product $\inner{\frac{\partial^{p} f(\rmW;\vx)}{\partial W^{(a_1)}\cdots \partial W^{(a_{p})}}, \rmV_p}$ can be equally viewed as $W^{(a_1)},\cdots,W^{(a_p)}$ being replaced by $V_1,\cdots,V_p$ respectively in the expression of $f(\rmW;\vx)$. Specifically,
\begin{align*}
    \innert{\frac{\partial^{p} f(\rmW;\vx)}{\partial W^{(a_1)}\cdots \partial W^{(a_{p})}}, \rmV_p} = f((W^{(1)},\cdots,W^{(a_1-1)},V_1,W^{(a_1+1)},\cdots,W^{(L)});\vx),
\end{align*}
where we for simplicity of notation assume $W^{(1)}$ and $W^{(L)}$ are not differentiated, but in general they can.

Therefore, we can bound the norm simply by: 
\begin{align*}
     \opnorm{\frac{\partial^{p} f(\rmW;\vx)}{\partial W^{(a_1)}\cdots \partial W^{(a_{p})}}}&\leq \left(\frac{1}{\sqrt{m}}\right)^{L}\frac{1}{\sqrt{d}}\prod_{\ell\in[L-1], \ell \notin \{a_1,\cdots, a_p\}}\opnorm{W^{(\ell)}}\prod_{i \in \{a_1,\cdots, a_p\}}\sup_{\|V_i\|_F=1}\opnorm{V_i}\|\vx\|\\
    &\leq \left(\frac{1}{\sqrt{m}}\right)^{L-1}\frac{1}{\sqrt{d}}\prod_{\ell\in[L],\ell \notin \{a_1,\cdots, a_p\}}\opnorm{W^{(\ell)}}\|\vx\|.
\end{align*}

By lemma~\ref{lemma:random_matrix}, for each $\ell \in [L]$, we have with probability at least   $1-2 \exp(-m/2)$,
\begin{align*}
    \opnorm{W^{(\ell)}} =  \opnorm{W_{\init}^{(\ell)} + W^{(\ell)}- W_{\init}^{(\ell)}} \leq  \opnorm{W_{\init}^{(\ell)}} +  \opnorm{W^{(\ell)}- W_{\init}^{(\ell)}} \leq 3\sqrt{m} + R.
\end{align*}
Here we use $m>d$ which is true in our setting.

Therefore, by union bound over indices $\ell\in[L-1]$ and $\ell \notin \{a_1,\cdots, a_p\}$, we have with probability at least $1-2(L-p)\exp(-m/2)$,
\begin{align*}
    \opnorm{\frac{\partial^{p} f(\rmW;\vx)}{\partial W^{(a_1)}\cdots \partial W^{(a_{p})}}}\leq \frac{(3\sqrt{m}+R)^{L-p}\|\vx\|}{(\sqrt{m})^{L-1} \sqrt{d}}.
\end{align*}

\end{proof}

\section{Technical lemmas}
\begin{lemma}[Tail bound for Gaussian random variable (Eq.~(2.9) in \cite{rigollet2015high})]\label{lemma:gaussian}
If $z \sim \mathcal{N}(0,\sigma^2)$, then for any $t>0$,
\begin{align*}
    \mathbb{P}\left[|z|\geq t\right]\leq 2e^{-\frac{t^2}{2\sigma^2}}.
\end{align*}
\end{lemma}

\begin{lemma}[Tail bound for $\chi^2$ distribution (Example~2.11 in \cite{rigollet2015high})]\label{lemma:chi_2}
If $z \sim \chi^2(m)$ for $m>0$, then for any $t\in(0,1)$,
\begin{align*}
    \mathbb{P}\left[\left|\frac{z}{m}-1\right|\geq t\right]\leq 2e^{-mt^2/8}.
\end{align*}
\end{lemma}

\begin{lemma}[Corollary 5.35 from~\cite{vershynin2010introduction}]\label{lemma:random_matrix}
Let $A$ be an $N \times n$ matrix whose entries are independent standard normal random variables. Then for every $t \geq 0$, with probability at least $1-2 e^{-t^2/2}$ one has
\begin{align*}
    \opnorm{A}\leq \sqrt{n}+\sqrt{N}+t.
\end{align*}
\end{lemma}

\begin{lemma}[Norm of tensors]\label{lemma:tensor_norm}
Given a tensor $\mathcal{T}\in\mathbb{R}^{d_1\times d_2\times\cdots d_r}$, suppose there are $p$ subtensors of $\mathcal{T}$ such that $\sum_{k=1}^p \mathcal{T}_k = \mathcal{T}$. Then we have
\begin{align}
    \opnorm{\mathcal{T}} \leq \sum_{k=1}^p \opnorm{\mathcal{T}'_k},
\end{align}
where $\mathcal{T}'_k$ is upper left block of the rearrangement of the entries of $\mathcal{T}_k$ by elementary matrix operations so that all the entries of $\mathcal{T}_k$ are moved to the upper left corner. 

Specifically, $\begin{bmatrix} \mathcal{T}'_k &\zero \\
    \zero & \zero
    \end{bmatrix}$ is the output of elementary matrix operations on $\mathcal{T}_k$.
\end{lemma}
\begin{proof}[Proof of Lemma~\ref{lemma:tensor_norm}]
Since elementary matrix operations do not change the spectral norm of the matrix, we directly have 
\begin{align*}
    \opnorm{\mathcal{T}_k} = \opnorm{\begin{bmatrix} \mathcal{T}'_k &\zero \\
    \zero & \zero
    \end{bmatrix}}.
\end{align*}
And it is not hard to see that 
\begin{align*}
     \opnorm{\begin{bmatrix} \mathcal{T}'_k &\zero \\
    \zero & \zero
    \end{bmatrix}} = \opnorm{\mathcal{T}'_k}.
\end{align*}
Since $\sum_{k=1}^p \mathcal{T}_k = \mathcal{T}$, by Cauchy–Schwarz inequality, we finish the proof.

\end{proof}

\end{document}